\documentclass[12pt]{article}
\usepackage[margin=1in]{geometry}
\usepackage{microtype}  
\usepackage{amsthm}
\usepackage{amssymb}
\usepackage{color}
\usepackage{xcolor}
\usepackage{soul}
\usepackage{mathrsfs}
\usepackage[ruled,linesnumbered]{algorithm2e}
\usepackage{algpseudocode}
\usepackage{times}
\usepackage{amsmath}
\usepackage{amssymb}
\usepackage{graphicx}
\usepackage{hyperref}  

\DeclareMathOperator{\KL}{KL}

\def\Var{\mathsf{Var}}
\def\ep{\varepsilon}

\def\mR{\mathbb{R}}
\def\R{\mathbb{R}}

\def\sV{\mathcal{V}}
\def\sH{\mathcal{H}}

\def\sP{\mathcal{P}}
\def\E{\mathbb{E}}

\def\sN{\mathcal{N}}

\def\sF{\mathcal{F}}
\def\sH{\mathcal{H}}
\def\W{\mathcal{W}}
\def\Law{\mathsf{Law}}
\def\sE{\mathcal{E}}
\def\FD{\mathsf{FD}}
\def\P{\mathbb{P}}

\def\vI{\mathbf{I}}
\def\de{\mathsf{DE}}
\def\se{\mathsf{SE}}
\def\le{\mathsf{LE}}
\def\LSI{\mathsf{LSI}}

\newtheorem{assumption}{Assumption}
\newtheorem{theorem}{Theorem}
\newtheorem{lemma}{Lemma}
\newtheorem{remark}{Remark}

\newcommand{\sR}{\mathcal{R}}
\newcommand{\iidsim}{\stackrel{\text{i.i.d.}}{\sim}}
\newcommand{\innerrho}[3]{\left<#2,#3\right>_{L_2(#1)}}
\newcommand{\normrhosq}[2]{L_2^2(#1;#2)}

\makeatletter
\def\@makefnmark{\hbox{$^{\@thefnmark}$}}
\makeatother

\renewcommand\thefootnote{\ifcase\value{footnote}\or*\or\dagger\or\ddagger\or\S\or\P\or\|\fi}

\skip\footins=20pt  
\begin{document}

\begin{center}
{\Large Beyond Propagation of Chaos: A Stochastic \\ 
\vspace{0.3cm}
Algorithm for Mean Field Optimization}

\vspace{0.5cm}

Chandan Tankala$^*$ \hspace{1cm} Dheeraj M. Nagaraj$^\dagger$ \hspace{1cm} Anant Raj$^\ddagger$

\vspace{0.25cm}
March 14, 2025
\end{center}

\begin{abstract}%
Gradient flow in the 2-Wasserstein space is widely used to optimize functionals over probability distributions and is typically implemented using an interacting particle system with $n$ particles. Analyzing these algorithms requires showing (a) that the finite-particle system converges and/or (b) that the resultant empirical distribution of the particles closely approximates the optimal distribution (i.e., propagation of chaos). However, establishing efficient sufficient conditions can be challenging, as the finite particle system may produce heavily dependent random variables. 

In this work, we study the virtual particle stochastic approximation, originally introduced for Stein Variational Gradient Descent \cite{dasnagaraj2023}. This method can be viewed as a form of stochastic gradient descent in the Wasserstein space and can be implemented efficiently. In popular settings, we demonstrate that our algorithm's output converges to the optimal distribution under conditions similar to those for the infinite particle limit, and it produces i.i.d. samples without the need to explicitly establish propagation of chaos bounds.
\end{abstract}


\stepcounter{footnote}
\stepcounter{footnote}
\stepcounter{footnote}

\footnotetext[1]{University of Oregon, Eugene, chandant@uoregon.edu}
\footnotetext[2]{Google DeepMind, dheerajnagaraj@google.com}
\footnotetext[3]{Indian Institute of Science, anantraj@iisc.ac.in}

\newpage
\section{Introduction}
\label{sec:intro}

Optimizing a functional $\sE()$ over the space of all probability distributions over $\R^d$ with finite second moments ($\sP_{2}(\R^d)$) has gained immense interest in the recent years with applications in machine learning and Bayesian inference. A notable example is the training and analysis of neural networks in the infinite-width regime. While analyzing neural network training is challenging due to inherent non-linearity, the infinite-width limit—known as the mean-field limit—facilitates a more tractable analysis. In this regime, the optimization problem reduces to optimizing over the distribution of neuron weights to achieve accurate label prediction \cite{chizat2018global,mei2018mean,nitanda2022convex,suzuki2023uniform,nitanda2023improved}. As other important examples, the task of sampling can be re-cast as the optimization of the Kullback-Leibler (KL) divergence to the target distribution \cite{vempalawibisono,durmus2019analysis}, and variational inference involves constrained optimization over the space of distributions to fit the given data \cite{yao2022mean,lacker2023independent,liu2016stein,lambert2022variational,yan2024learning}. A prototypical example of such a functional is given by $\sE(\mu)= \int V ~d\mu + \int \log \mu ~d\mu$. More broadly, this work examines the following optimization objective. Given an energy functional  $\sF: \sP_2(\R^d) \to \R$, and regularization strength $\sigma > 0$, we consider functionals of the form $\sE: \sP_2(\R^d) \to \R$ defined as:
\begin{align}
    \sE(\mu) = \sF(\mu) + \frac{\sigma^2}{2}\sH(\mu) \,, \label{eq:mean_field_obj}
\end{align}
where $\sH(\mu)$ is the negative entropy defined as follows:
\begin{equation}
    \sH(\mu) = \begin{cases} \int \mu(x) \log \mu(x) dx \quad \text{if } \mu \ll \mathsf{Leb} \text{ and }d\mu(x) = \mu(x)dx \nonumber \\
    \infty \quad \text{otherwise} \,.
    \end{cases}
\end{equation}
A common approach to optimization over $\sP_2(\R^d)$
 is gradient flow with respect to the Wasserstein metric. The well-known Langevin dynamics was shown to be the gradient flow of the Kullback-Leibler (KL) divergence to the target distribution in the seminal work of \cite{jordan1998variational}. This framework can be extended to a broader class of functionals, including interaction energy and entropy \cite{mccann1997convexity,ambrosio2008metric}. While Langevin dynamics can be implemented algorithmically through time discretization of It\^o stochastic differential equations (SDE), the more general case—leading to McKean-Vlasov type SDEs \cite{carmona2015forward}, where the drift function depends on the distribution of the variable—is not straightforward. The popular computational approximation in this context is the particle approximation where $n$ instances (or particles) of the SDE are implemented computationally, with the distribution of the variables replaced by their empirical distribution. Theoretically, one way to show that this approximation optimizes the objective functional is to show that (a) the particle approximation converges rapidly to its $nd$ dimensional stationary distribution and (b) a sample from the $n$ particle stationary distribution gives a representative sample from the optimal distribution (called propagation of chaos). Propagation of chaos type of results show that whenever we pick any $k$ out of the $n$ particles, then the resulting $k$ particle distribution is close to the $k$ fold product distribution of the target with an error of $O(\mathsf{poly}(k/n))$ in metrics such as KL divergence.

To show the convergence of $n$ particle approximation, prior works such as \cite{chen2022uniform,chewi2024uniform,wang2024uniform} consider the finite-particle stationary distribution and establish Logarithmic Sobolev inequalities (LSI) for the $nd$ dimensional system with a constant independent of $n$. This can be used to establish computational complexity of standard sampling algorithms such as LMC, ULMC, and MALA, allowing us to obtain guarantees for a large class of sampling algorithms instead of specialized analysis for each. However, this can be technically involved and yield pessimistic estimates for the LSI which are worse than those established in the mean-field case (compare to \cite[Theorem 1]{suzuki2024mean}). Additionally, the complexity of sampling can be polynomial in the dimension $nd$, requiring more iterations with more particles (see \cite{vempalawibisono}).
Obtaining guarantees for the LSI constant $C_{\LSI, n}$ for the $n$-particle stationary distribution can be very hard and $C_{\LSI, n}$ could be much worse than the LSI constant $C_{\LSI}$ for the mean field stationary distribution $\pi$. For instance, in \cite{wang2024uniform}[Theorem 1], $C_{\LSI, n} = O(d C_\LSI^3)$ (along with additional assumptions), and in \cite{chewi2024uniform}[Theorem 2], $C_{\LSI, n} = O(e^d)$, while \cite{chewi2024uniform}[Equation 2.2], $C_{\LSI}$ does not depend on $d$.



\paragraph{Our Contribution}

We consider stochastic approximations of the mean field SDEs, which can be implemented exactly, when an unbiased estimator for the Wasserstein gradients is available (Algorithm~\ref{alg:main_algo}). The virtual particle stochastic approximation method, first proposed by \cite{dasnagaraj2023} for Stein Variational Gradient Descent, is the theoretical foundation of our analysis. We expand its applicability to sampling from the stationary distributions of McKean-Vlasov SDEs. This extension presents novel analytical challenges, particularly in quantifying the approximation of the non-linear stationary probability measure induced by the McKean-Vlasov SDE and the Brownian diffusion process, which presents almost everywhere non differentiable paths. The algorithm , which outputs $n$ particles for $T$ iterations/time steps, has the following properties:

\begin{enumerate}
    \item \textbf{Computational complexity} $O(nT + T^2)$, unlike the $O(n^2 T)$ complexity of standard particle methods. The \textbf{output particles are i.i.d.} from a distribution close to the minimizer of $\sE$, and does not require us to establish propagation of chaos separately. 
    \item Our general result in Theorem~\ref{thm:main_theorem} shows that the number of time steps $T$ required to sample from an $\epsilon$ optimal distribution is polynomial in the problem parameters, depends on the \textbf{isoperimetry constant of only the mean field dynamics} and is independent of $n$. As noted above, the isoperimetry bounds for the $n$ particle system, $C_{\mathsf{LSI},n}$, can be much worse than that of the mean field optimal distribution $C_{\mathsf{LSI}}$
    \item We illustrate our result in two important scenarios: pairwise interaction energy in the weak interaction regime (Equation~\eqref{eq:int_energy}) and the mean field neural network with square loss (Equation~\eqref{eq:mean_field}), we establish these results under standard assumptions.
\end{enumerate}

We note that prior works can deal with a much larger class of functionals \cite{chen2022uniform,wang2024uniform,nitanda2022convex}, without access to unbiased estimators for the gradients.

\subsection{Prior Work}

\paragraph{Stochastic Approximation and Sampling Algorithms:} 

Stochatic Gradient Langevin Dynamics \cite{welling2011bayesian} was introduced as a stochastic, computationally viable variant of Langevin Monte Carlo (LMC) and has been extensively studied in the literature \cite{raginsky2017non,kinoshita2022improved,dasnagarajraj2023}. This has been extended to other sampling algorithms and interacting particle systems \cite{huang2024faster,jin2020random}. Recently, virtual particle stochastic approximation \cite{dasnagaraj2023} was introduced in the context of Stein Variational Gradient Descent (SVGD) \cite{liu2016stein}, where the algorithm directly produced an unbiased estimate of the flow in the space of probability distributions, giving the first provably fast finite particle variant of SVGD. Note that such bounds have been obtained for the traditional SVGD algorithm since then \cite{balasubramanian2024improved}. Stochastic approximations have also been utilized to obtain speedup of sampling algorithm with randomized mid-point based time discretization \cite{kandasamy2024poisson,yu2023langevin,shen2019randomized}.

\paragraph{Propagation of Chaos:} 
The propagation of chaos problem for McKean-Vlasov SDEs was originally studied by \cite{sznitman1991topics}, which established convergence rates in the Wasserstein metric via coupling arguments. These bounds were first made uniform in time by \cite{malrieu2001logarithmic, malrieu2003convergence} in the quadratic Wasserstein and relative entropy metrics. In the case of pairwise interaction energy (Equation~\eqref{eq:int_energy}), these works obtain a bound on the error of order $O(k/n)$ in the squared quadratic Wasserstein distance and assume strong convexity for the \emph{external potential} and convexity for the \emph{interaction potential}. A uniform in time propagation of chaos was recently shown by \cite{chen2022uniform} by assuming convexity of the mean-field functional, as opposed to imposing convexity conditions on the interaction potential. The error in the squared quadratic Wasserstein error bound was improved to $O((k/n)^2)$ in \cite{lacker2023sharp} by assuming a uniform-in-$n$ log-Sobolev inequality for the stationary distribution of the $n$-particle system and using the recursive BBGKY proof technique. \cite{kook2024sampling} build on this proof technique and also obtain an error bound of order $O((k/n)^2)$ in the squared quadratic Wasserstein idstance and KL-divergence under a slightly different assumption on a ratio involving log-Sobolev inequality, smoothness and diffusion constants, which is referred to as the ``weak interaction'' condition. Recently, \cite{bou2023nonlinear} presented a non-linear Hamiltonian Monte Carlo algorithm and prove its rate of convergence in $L^1$-Wasserstein distance, without using the propagation of chaos arguments. 

\paragraph{Mean Field Optimization:} Mean-field analysis of neural networks emerged as a theoretical framework for understanding the optimization dynamics of wide neural networks. Early foundational work by \cite{nitanda2017stochastic, chizat2018global, mei2018mean} established that gradient flow on infinite width two-layer neural networks converges to the global minimum under appropriate conditions, demonstrating that we can successfully study neural networks in the infinite-dimensional space of parameter distributions by exploiting convexity. The connection with the mean-field Langevin dynamics arises with the addition of Gaussian noise to the gradient, corresponding to the entropy-regularized term in the objective function. \cite{nitanda2022convex, chizat2022mean} were among the first to establish exponential convergence rates under certain LSIs, which are verifiable in regularized risk minimization problems using two-layer neural networks. Subsequently, \cite{suzuki2023uniform,suzuki2024mean} study uniform in time propagation of chaos result in the context of mean-field neural networks where the main ingredient is the proximal Gibbs distribution, which also satisifies a LSI for convex losses with smooth and bounded activation functions. 

\paragraph{Wasserstein Gradient Flows:} These describe the evolution of probability measures over \( \mathcal{P}_2(\mathbb{R}^d) \) equipped with the Wasserstein-2 metric. Given an energy functional \( F: \mathcal{P}_2(\mathbb{R}^d) \to \mathbb{R} \), the gradient flow $\mu_t : \R^{+} \to \sP_2(\R^d)$ is formally the solution to the evolution equation 
\small
\[
\frac{d}{dt} \mu_t = -\text{grad}_W F(\mu_t),
\]
where \( \text{grad}_W F(\mu) \) is the Wasserstein gradient of \( F \).
\normalsize
More rigorously, the gradient flow $\mu_t \in \sP_2(\R^d)$ is a curve in $\sP_2(\R^d)$ which satisfies the following \emph{continuity equation} in the sense of distributions:
\[
    \frac{\partial \mu_t(x)}{\partial t} + \nabla_x \cdot (\mu_t(x) v_t(x)) = 0 \, ,
\]
\sloppy
where the `velocity field' $v_t(x): \mathbb{R}^+\times \mathbb{R}^d  \to \mathbb{R}^d = - \nabla_\W \sF(\mu_t)(x)$ and $\nabla_\W \sF(x,\mu_t)$ is often also called the Wasserstein gradient. Under certain regularity conditions, which hold in all the cases considered in this work, $\nabla_\W \sF(\mu_t)(x) = \nabla_x \frac{\delta F}{\delta \mu}(x,\mu_t) $. Here \( \frac{\delta F}{\delta \mu} \) denotes the first variation (Eulerian derivative) of \( F \). 
Important cases include the heat equation (when \( F \) is the entropy functional) and the Fokker-Planck equation (when \( F \) includes an external potential term). These flows provide a geometric perspective on evolution equations in probability spaces with applications in statistical physics and PDEs. For a rigorous treatment, see \cite{ambrosio2008metric,villani2021OptTransport}. 


\subsection{Notation}
For any measure, $\rho$ over $\R^d$ and functions $f,g :\R^d \to \R^d$. Let $\innerrho{\rho}{f}{g} := \int \rho(dx)\langle f(x),g(x)\rangle$ and $\normrhosq{\rho}{f}^2 := \int \rho(dx)\|f(x)\|^2$ whenever $f,g$ are square integrable with respect to $\rho$. For a vector field $f:\mR^d \to \mR^d$, its divergence is given by $\nabla \cdot f = \sum_{i = 1}^d \frac{\partial f}{\partial x_i}$, and for a function $f:\mR\to \mR$, the Laplacian is defined as $\Delta f:= \sum_{i=1}^d \frac{\partial^2 f}{\partial x_i^2}$.
Let $\sP_2(\mathbb{R}^d), \sP_{2,\mathsf{ac}}(\mR^d)$ denote the space of probability measures on $\mathbb{R}^d$ with finite second moment, and those that are absolutely continuous with respect to the Lebesgue measure. For $\mu \in \sP_2(\R^d)$, we let $\Var(\mu)$ denote the trace of its covariance. The Wasserstein distance $\W_2(\mu,\nu)$ between two probability measures $\mu, \nu \in \sP_2(\mR^d)$ is defined as:
\small
\[
    \W_2^2(\mu,\nu) := \inf_{\gamma \in \Gamma(\mu,\nu)} \int_{\mathbb{R}^d \times \mathbb{R}^d} \|X-Y\|^2 d\gamma(x,y) \,,
\]
\normalsize
where $\Gamma(\mu,\nu)$ is the set of all joint distributions over $\R^d\times\R^d$ such that the marginal distribution of $X$ is $\mu$ and of $Y$ is $\nu$. The Fisher Divergence of a probability measure $\mu \in \sP_2(\mR^d)$ with respect to $\nu\in \sP_2(\mR^d)$ is defined as: $\FD(\mu || \nu) := \int_{\mR^d} \mu(x) \big\| \nabla \log \frac{\mu(x)}{\nu(x)} \big\|^2dx $. The first variation of a functional $\sF$ at $\mu\in \sP_2(\mR^d)$ is denoted by $\delta_\mu \sF(\mu)(x)$ or just $\delta \sF(x,\mu)$, where $x\in \mR^d$, and is defined as the quantity which satisfies the equality:
\small
\[
    \frac{d\sF(\mu + \ep(\mu' - \mu))}{d\ep} \bigg|_{\ep = 0} = \int \delta_\mu \sF(\mu)(x)(\mu' - \mu)(x)dx \,.
\]
\normalsize
If the probability measures $\mu,\nu \in \sP_2(\mR^d)$ have densities $p,q$ respectively, then the total-variation distance between them is defined as: $
    \| \mu - \nu \|_{\mathsf{TV}} := \frac{1}{2} \int_{\mR^d} |p(x) - q(x)|dx$.

\section{Preliminaries and Problem Setup}

In this work, we consider the energy functional given in equation~\eqref{eq:mean_field_obj} as $\sE(\mu) = \sF(\mu) + \frac{\sigma^2}{2}\sH(\mu) \,,$ where $\sH(\mu)$ is the negative entropy.
Under mild conditions on $\sF$ \cite{ambrosio2008metric}, given $\mu_0 \in \sP_2(\R^d)$, the gradient flow of this functional with respect to the Wasserstein metric is the curve $\mu : [0,T] \to \sP_2(\R^d)$  given by the following non-linear Fokker-Planck Equation:
\small
\begin{equation}
    \frac{\partial \mu_t}{\partial t} - \nabla \cdot (\nabla_{\W}\sF(x,\mu_t)\mu_t) = \frac{\sigma^2}{2}\Delta \mu_t \,.
\end{equation}
\normalsize
\sloppy
Here $\nabla_{\W}\sF(\cdot,\cdot) : \R^d \times \sP_2(\R^d) \to \R^d $ is the Wasserstein gradient. In our applications, we can show that $\nabla_{\W}\sF = \nabla \delta \sF$, where $\delta \sF$ is the first variation of $\sF$.  This is considered non-linear since the drift depends on $\mu_t$. The SDE corresponding to this is called the McKean-Vlasov equation given by:
\small
\begin{equation}\label{eq:mckean_vlasov}\tag{MKV}
dX_t = -\nabla_\W\sF(X_t,\mu_t) +\sigma dB_t \,;\quad \mu_t := \Law(X_t) \,;\quad X_0 \sim \mu_0 \,.
\end{equation}
\normalsize

\eqref{eq:mckean_vlasov} can be seen as a sampling algorithm whenever it converges to a stationary distribution which is same as the global minimizer of $\sE(\mu)$. The drift in~\eqref{eq:mckean_vlasov} being dependent on $\mu_t$ makes it hard to approximate the dynamics algorithmically. Therefore particle approximation is used: let $X_0^{(1)},\dots,X_0^{(n)} \iidsim \mu_0$ be i.i.d. $\hat{\mu}_k $ be the empirical distribution of $(X_k^{(i)})_{i \in [n]}$. Let $\eta$ be the timestep size and $(Z_k^{(i)})_{k,i} \iidsim \sN(0,\vI)$. The particle approximation is then given by:
\small
\begin{equation}\label{eq:particle_mkv}\tag{pMKV}
    X_{k+1}^{(i)} = X_k^{(i)} - \eta\nabla_\W \sF(X_k^{(i)},\hat{\mu}_k) + \sigma \sqrt{\eta} Z_k^{(i)} \,;\quad \forall \ i\in [n] \,.
\end{equation}
\normalsize
\paragraph{Optimization, Dynamics and Stochastic Approximation:}
We motivate our setting via an analogy with gradient descent (GD) over $\R^d$, when the objective function is $F(x) = \E_{\theta \sim P}f(x,\theta)$ and only samples $\theta_1,\dots, \theta_N \iidsim P$ are available. One approach is to consider the minimization of empirical risk $\hat{F}(x) := \frac{1}{N}\sum_{i=1}^{N} f(x,\theta_i)$ via GD: $x_{t+1} = x_t - \alpha \nabla\hat{F}(x_t)$. However this can be a) computationally expensive and b) need not converge to $\arg\min_x F(x)$. This can lead to sub-optimal convergence of GD \cite{amir2021sgd}. GD's stochastic approximation, stochastic gradient descent (SGD), updates $x_{t+1} = x_t - \alpha \nabla f(x_t;z_t)$ for $t  = 1,\dots, N$. Since $\E[f(x_t;z_t)|x_t] = \nabla F(x_t)$, SGD with one pass over the data minimizes $F$ (not $\hat{F}$) \cite{polyak1992acceleration,chen2020finite,robbins1951stochastic,kushner2003stochastic,godichon2021non} with near optimal convergence in problems of interest, making it popular in machine learning.

The stochastic approximation algorithm we introduce below directly approximates the population level dynamics in~\eqref{eq:mckean_vlasov} directly whenever we have an estimator $\hat{G}(x,y,\xi) :\R^d \times \R^d \times \R^m \to \R^d $ such that $\E_{Y\sim \mu,\xi \sim \nu}\hat{G}(x,Y,\xi) = -\nabla_{\W}\sF(x,\mu)$ for every $\mu \in \sP_2(\R^d)$ and for a given $\nu \in \sP_2(\R^d)$. We refer to Section~\ref{subsec:examples} for a rich class of functionals $\sE$ with such an estimator.

\subsection{The Virtual Particle Stochastic Approximation}
Given a McKean-Vlasov type SDE of the form $dX_t = G(X_t,\mu_t)dt + \sigma dB_t$, where $\mu_t = \Law(X_t)$, $G: \R^d \times \sP_2(\R^d) \to \R^d$ is the drift and $B_t$ is the standard Brownian motion. The Euler-Maruyama discretization of this is given by:
$X_{k+1} = X_k + \eta G(X_k,\mu_k) + \sigma\sqrt{\eta}Z_k$.
Here $(Z_k)_k$ are i.i.d. $\sN(0,\vI)$ random variables, $\eta$ is the timestep size, and $\mu_k = \Law(X_k)$. However, this is not tractable since we do not know the distribution $\mu_k$. Thus, we introduce the tractable approximation, called the virtual particle stochastic approximation (Algorithm~\ref{alg:main_algo}), a generalization of VP-SVGD \cite{dasnagaraj2023}. Suppose there exists an estimator $\hat{G}: \R^d \times \R^d \times \R^m \to \R^d$ such that whenever $Y,\xi \sim \mu\times \nu$  then $\E[\hat{G}(x,Y,\xi)] = G(x,\mu)$ for every $x\in \R^d $, $\mu \in \sP_2(\R^d)$ and $\nu$ is a fixed, known distribution. Let the initial distribution be $\mu_0$. Fix number of timesteps $T$, number of desired samples $n$ and timestep size $\eta$.

\vspace{0.5cm}

\SetKwComment{Comment}{/* }{ */}
\begin{algorithm}[H]
\DontPrintSemicolon
\caption{Virtual Particle Stochastic Approximation}
\label{alg:main_algo}
\textbf{Data:} Time steps $T$, number of samples $n$, timestep size $\eta$. Initial Distribution $\mu_0$, distribution $\nu$, estimator $\hat{G}$ \;
\textbf{Result:} $X_T^{(1)},\dots,X_T^{(n)}$ \;

$X_0^{(1)},\dots,X_0^{(n)},Y_0^{(0)},\dots,Y_0^{(T)} \stackrel{\text{i.i.d.}}{\sim} \mu_0$ \;
$\xi_1,\dots,\xi_T \stackrel{\text{i.i.d.}}{\sim} \nu$ \;
$k\gets 0$ \;

\While{$k \leq T$}{
    \For{$i = 1,..,n$}{
        $X_{k+1}^{(i)} = X_k^{(i)}+ \eta \hat{G}(X^{(i)}_k,Y_k^{(k)},\xi_k) + \sigma \sqrt{\eta}Z_k^{(i)}$ \quad $Z_k^{(i)}\stackrel{\text{i.i.d.}}{\sim} \mathcal{N}(0,\mathbf{I})$ \;
    }
    \For{$j = k+1,...T,$}{
        $Y_{k+1}^{(j)} = Y_k^{(j)} + \eta \hat{G}(Y^{(j)}_k,Y^{(k)}_k,\xi_k) + \sigma \sqrt{\eta}W_k^{(j)}$ \quad $W_k^{(j)} \stackrel{\text{i.i.d.}}{\sim} \mathcal{N}(0,\mathbf{I})$ \; 
    }
}
\end{algorithm}

At timestep $k$, the particle $Y_k^{(k)}$ is used to estimate $\mu_k$ for all the ``real'' particles $X^{(i)}_k$, and the remaining ``virtual'' $Y_k^{(j)}$ are then discarded. Let $\sR_k$ be the sigma algebra of $Y_0^{(0)},\dots,Y_k^{(k)},\xi_0,\dots,\xi_k$. Then the following can be easily demonstrated:

\begin{lemma}
Conditioned on $\sR_{k-1}$, $X_k^{(i)}$,$Y_k^{(j)}$ for $i = 1,\dots, k$ and $j = k,\dots,T$ are i.i.d. 
\end{lemma} 
Let $\mu_{k}|\sR_{k-1}$ be the law of $X_k^{(1)}$ conditioned on $\sR_{k-1}$. It is a random probability measure which is measurable with respect to $\sR_{k-1}$.

\paragraph{Witness Path:}In algorithm~\ref{alg:main_algo}, we call the diagonal trajectory $Y_{k}^{(k)}$ for $0\leq k \leq T$ as the `witness path'. Notice that given a `witness path, we can obtain a sample from $\mu_T|\sR_{T-1}$ in $T$ steps, without fixing $n$ beforehand. Therefore, storing the witness path yields an approximate sampling algorithm from the global minima $\pi$ of $\sE$ whenever the Markov process Equation~\eqref{eq:mckean_vlasov} converges to $\pi$. 

\paragraph{Computational Complexity:} Algorithm~\ref{alg:main_algo} produces $n$ i.i.d. samples from $\mu_T|\sR_{T-1}$, which is shown to converge to $\pi$ in Theorem~\ref{thm:main_theorem}. The computational complexity is $O(nT + T^2)$, avoiding the $O(n^2T)$ complexity incurred in the straight forward particle approximation \eqref{eq:particle_mkv}. Additionally Theorems~\ref{thm:main_theorem},~\ref{thm:main_pairwise} and~\ref{thm:main_mean_field} show that we can pick $T$ independent of the $n$ to ensure $\epsilon$ error.

\subsection{Examples}
\label{subsec:examples}
\paragraph{Example 1: Pairwise Interaction Energy}
Let $V,W: \mR^d \to \mR$ and $\mu\in \sP_{2, \mathsf{ac}}(\mR^d)$. In this case, $V$ is commonly referred to as the external potential and $W$ as the interaction potential. Let $W$ be even (i.e, $W(x) = W(-x)$). The functional $\sE$ in this case is defined as:
\small
\begin{equation}\label{eq:int_energy}
    \sE(\mu) := \int V(x)d\mu(x) + \tfrac{1}{2}\int \int W(x-y)d\mu(x)d\mu(y) + \tfrac{\sigma^2}{2}\sH(\mu) \,.
\end{equation}
\normalsize
Here the Wasserstein gradient flow gives the following McKean-Vlasov dynamics
$dX_t = -\nabla V(X_t)dt - (\nabla W \ast \mu_t)(X_t) + \sqrt{\sigma} dB_t$ where $\mu_t = \Law(X_t)$.
From \cite{ambrosiosavarenotes}[Proposition 4.13], for the potential energy functional $\sV(\mu) := \int V(x) \mu(dx)$, the Wasserstein gradient is given by $\nabla_\W \sV(\mu) = \nabla V$. Next, from \cite{ambrosiosavarenotes}[Theorem 4.19], for the interaction energy functional $\mathscr{W}(\mu) := \frac{1}{2}\int\int W(x-y)\mu(dx)\mu(dy)$, the Wasserstein gradient is given by
$\nabla_{\W}\mathscr{W}(\mu) = (\nabla W) \ast \mu$ if $(\nabla W) \ast \mu \in L^2(\mu;\mR^d)$. Finally, from \cite{ambrosiosavarenotes}[Theorem 4.16], for the entropy functional
$\sH(\mu) := \int \log \mu d\mu$, the Wasserstein gradient is given by $\nabla_{\W} \sH(\mu) = \nabla \log \mu$ if $\nabla \log \mu \in L^2(\mu;\,\mR^d)$. Thus,
\small
\begin{equation}
\label{Eq:IEWassGradient}
\nabla_{\W} \sE(\mu) = \nabla V + (\nabla W)\ast \mu + \tfrac{\sigma^2}{2}\nabla \log \mu \,.
\end{equation}
\normalsize
The unique minimizer of the functional $\sE$ satisfies a fixed point equation and is given in \cite{kook2024sampling}[Equation 1.1] as:
\small
\begin{equation} \label{Eq:IEstationarymeasure}
    \pi(x) \propto \exp\left( -\tfrac{2}{\sigma^2} V(x) - \tfrac{2}{\sigma^2} W \ast \pi(x) \right) \,.
\end{equation}
\normalsize

\paragraph{Example 2: Mean Field Neural Network}

Let $\mu\in \sP_{2,\mathsf{ac}}(\mR^d)$. Consider the activation function $h(x,z): \R^d\times\R^k \to \R$. We consider the two layer mean-field neural network to be $f(\mu;z) := \int h(x,z)d\mu(x)$. Given data $(Z,W)\sim P$, we consider the square loss:

\begin{equation}\label{eq:mean_field}
\sE(\mu) = \E_{(Z,W)}(f(\mu;Z)-W)^2 + \frac{\lambda}{2} \int \|x\|^2 d\mu(x) + \frac{\sigma^2}{2}\sH(\mu) \,.
\end{equation}
If we have samples $(z_1,w_1),\dots,(z_n,w_n)$, then we take $P$ to be the empirical distribution. From \cite{nitanda2022convex}[Section 2.3], the first variation of the functional $\sE$ defined above is:
\small
\[
    \delta \sE(x;\mu) = \E_{(Z,W)} \left[2 (f(\mu;Z)-W) h(x,Z)] + \frac{\lambda}{2} \|x\|^2 \right] + \tfrac{\sigma^2}{2}(\log \mu + 1) \,,
\]
\normalsize
and since $\nabla_\W \sE = \nabla \delta \sE$, we have:
\small
\begin{equation}
\label{Eq:MFNWassGradient}
    \nabla_\W \sE(\mu) = 2\E_{(Z,W)} \left[ (f(\mu;Z)-W) \nabla_x h(x,Z)] + \lambda x \right] + \tfrac{\sigma^2}{2}\nabla \log \mu \,.
\end{equation}
\normalsize
The unique minimizer of the functional $\sE$ is given in \cite{nitanda2022convex}[Equation 15] as the solution of the fixed point equation:
\small
\begin{equation}\label{eq:mf_optimal}
        \pi (x) \propto \exp\left( -\tfrac{2}{\sigma^2} \delta \sF(x,\pi) \right) \,.
\end{equation}
\normalsize

\section{Results}
In this section, we first establish a convergence theory for Algorithm \ref{alg:main_algo}. We begin with a key descent lemma (Lemma \ref{lem:descent_lemma}) that bounds the evolution of the energy functional by decomposing the error into discretization, stochastic, and linearization terms. Using this, we prove our main result (Theorem \ref{thm:main_theorem}) showing that Algorithm \ref{alg:main_algo} produces i.i.d. samples converging to the minimizer of the energy functional, with rates independent of the number of particles $n$. We then demonstrate how this framework applies to two important cases: pairwise interaction energy (Section \ref{subsec:pairwise}) and mean-field neural networks (Section \ref{subsec:mean_field}). In both cases, we verify the assumptions of Theorem \ref{thm:main_theorem} and establish explicit convergence rates while avoiding the need for separate propagation of chaos bounds.

\subsection{General Convergence:}
For some functional $\bar{\sF}:\sP_2(\R^d) \to \R$, let $\pi$ be the unique minimizer of the functional $\bar{\sF} + \frac{\sigma^2}{2}\sH$. Define $\bar{\sE}(\mu) := \bar{\sF}(\mu) + \frac{\sigma^2}{2}\sH(\mu) - \bar{\sF}(\pi) - \frac{\sigma^2}{2}\sH(\pi)$ (not necessarily $\sE$). The functional $\bar{\sE}$ is introduced, rather than simply analyzing $\sE$ since $\bar{\sE}$ can have better contraction properties. This is indeed the case with pairwise interaction energy where the KL functional to the minimizer $\pi$ has a contraction whenever $\pi$ satisifes an LSI. 

We consider Algorithm~\ref{alg:main_algo}, with $\hat{G}$ such that whenever $(Y,\xi) \sim \mu\times \nu$, we have:
$\E[\hat{G}(x,Y,\xi)] = -\nabla_{\W}\sF(x,\mu)$ for every $x \in \R^d$, $\mu \in \sP_2(\R^d)$. We do not assume $\sF \neq \bar{\sF}$, which is important for the case of the interaction energy. We track the evolution of $\E\bar{\sE}(\mu_k|\sR_{k-1})$ along the discrete time trajectory $\mu_k|\sR_{k-1}$ via continuous interpolations. We then specialize to the case of pairwise interaction energy (Equation~\eqref{eq:int_energy}) in Section~\ref{subsec:pairwise} and mean field neural networks (Equation~\eqref{eq:mean_field}) in Section~\ref{subsec:mean_field}, allowing us to obtain convergence bounds of Algorithm~\ref{alg:main_algo} for these cases. 

To simplify notation, consider $X_0,Y$ i.i.d. from a distribution $\rho_0 \in \sP_{2,\mathsf{ac}}(\R^d)$ and $\xi \sim \nu$ independent of $X_0,Y$. For $t \in [0,\eta]$, we consider the random variable $X_t := X_0 + tu(X_0,Y,\xi) + \sigma B_t$, where $u: \R^d\times\R^d \times \R^m \to \R^d$ is a velocity field and $B_t$ is the standard $\R^d$ Brownian motion independent of everything else. Let $\rho_t(\cdot|Y,\xi) := \Law(X_t |Y,\xi)$. Assume that $\E_{Y,\xi}[u(x,Y,\xi)] = -\nabla_\W \sF(x,\rho_0) $ for every $x \in \R^d$ and that $u(x,Y,\xi)$ has a finite second moment when $x \sim \rho_t(\cdot |Y,\xi)$ almost surely $Y,\xi$. Here $\rho_0,X_0,Y,\xi$ corresponds to $\mu_{k}|\sR_{k-1}, X^{(1)}_k,Y_k^{(k)},\xi_k$ respectively. The velocity field $u(x,Y,\xi)$ corresponds to $\hat{G}(x,Y,\xi)$. We use the notation $u$ and $\hat{G}$ interchangeably. Under this correspondence, it is clear that $X_{\eta}|Y,\xi$ has the same distribution as $\mu_{k+1}|\sR_{k}$. Following the proof of \cite[Lemma 3]{vempalawibisono}, we conclude that $\rho_t$ satisfies the Fokker-Planck equation:

\begin{align}
    \frac{\partial \rho_t(x|Y,\xi)}{\partial t} &= - \nabla_x . (\rho_t(x|Y,\xi)\E[u(X_0,Y,\xi)|X_t = x,Y,\xi]) + \frac{\sigma^2}{2}\Delta_x \rho_t(x|Y,\xi) \nonumber \\
    &= - \nabla_x . (\rho_t(x|Y,\xi)v_t(x,Y,\xi)) \,, 
\end{align}
where $v_t(x,Y,\xi) = \E[u(X_0,Y,\xi)|X_t = x,Y,\xi]-\frac{\sigma^2}{2}\nabla \log \rho_t(x|Y,\xi)$, $\forall x \in \R^d, t \in [0,\eta]$. Taking $\nabla_\W \bar{\sE}(x,\rho_t(\cdot |Y,\xi)) = \nabla_\W \bar{\sF}(x,\rho_t(\cdot|Y,\xi)) +\frac{\sigma^2}{2}\nabla \log \rho_t(x|Y,\xi)$, we have the following evolution of $\bar{\sE}$ along the trajectory $\rho_t(\cdot|y,\xi)$. 

\begin{lemma}\label{lem:error_decomp}
 $\rho_0\times \nu$ almost surely $(y,\xi)$, suppose that the energy functional $\bar{\sE}(\rho_t(\cdot|y,\xi))$ satisfies:
\begin{align*}
    \nabla_{\W}\bar{\sE}(,\rho_t(\cdot|y,\xi)) \in L^2(\rho_t(\cdot|y,\xi));\quad\frac{d\bar{\sE}(\rho_t(\cdot|y,\xi))}{dt} = \innerrho{\rho_t(\cdot | y,\xi)}{\nabla_\W \bar{\sE}(\cdot,\rho_t(\cdot|y,\xi))}{v_t(\cdot,y,\xi)} 
\end{align*}
Then, 
\small
\begin{align}\label{eq:time_evolution}
    \frac{d\bar{\sE}(\rho_t(\cdot |Y,\xi))}{dt} &= - \int d\rho_t(x|Y,\xi)\|\nabla_\W \bar{\sE}(x,\rho_t(\cdot |Y,\xi))\|^2 + \de_1(t) + \de_2(t) + \se(t) + \le(t)
\end{align}
\normalsize

\begin{enumerate}
\item \small $\de_1(t) := \innerrho{\rho_t(\cdot |Y,\xi)}{\nabla_\W \bar{\sE}(\cdot,\rho_t(\cdot |Y,\xi))}{\E[u(X_0,Y,\xi)|X_t = \cdot,Y,\xi] - u(\cdot,Y,\xi)}$\normalsize
\item \small $\de_2(t) := \innerrho{\rho_t(\cdot |Y,\xi)}{\nabla_\W \bar{\sE}(\cdot,\rho_t(\cdot |Y,\xi))}{ \nabla_\W \sF(\cdot,\rho_t(\cdot |Y,\xi)) -\nabla_\W \sF(\cdot,\rho_0)} $\normalsize
\item \small $\se(t) := \innerrho{\rho_t(\cdot|Y,\xi)}{\nabla_\W \bar{\sE}(\cdot,\rho_t(\cdot |Y,\xi))}{u(\cdot,Y,\xi) + \nabla_\W\sF(\cdot,\rho_0)} $ \normalsize
\item \small $\le(t) := \innerrho{\rho_t(\cdot |Y,\xi)}{\nabla_\W \bar{\sE}(\cdot,\rho_t(\cdot |Y,\xi))}{\nabla_\W \bar{\sF}(\cdot,\rho_t(\cdot |Y,\xi)) - \nabla_\W \sF(\cdot,\rho_t(\cdot |Y,\xi)} $ \normalsize
\end{enumerate}
\end{lemma}

\begin{remark}
    Here, $\de_1(t),\de_2(t),\se(t),\le(t)$ are all functions of $(Y,\xi)$. The \textbf{discretization error} $\de_1$ arises since $X_0$ is used instead of $X_t$ in $u(X_0,Y,\xi)$ and $\de_2$ arises since $- \E[u(x,Y,\xi)] = \nabla_{\W}\sF(x,\rho_0) \neq \nabla_{\W}\sF(x,\rho_t(\cdot |Y,\xi))$. $\se$ is the \textbf{stochastic error}, since we are using an estimator $u(\cdot,Y,\xi)$ of $\nabla_{\W}\sF(\cdot,\rho_0)$.  $\le$ is the \textbf{linearization error} since we consider the evolution of $\bar{\sE}$ where as the gradient descent is for $\sE \neq \bar{\sE}$. This is important for linearizing the non-linear Fokker-Planck equation, in the case of pairwise interaction energy in Section~\ref{subsec:pairwise}. 
\end{remark}
The Lemma below, proved in Section~\ref{sec:errorbounds_proof}, bounds each of the error terms in expectation. 
\begin{lemma} 
\label{lem:errorbounds}
In the setting of Lemma~\ref{lem:error_decomp},
let $\pi$ be the unique global minimizer of $\bar{\sE}$. Let $(Y^*,\xi) \sim \pi\times \nu$. Define $(\sigma^*)^2 := \E_{x\sim \pi^*}\mathsf{Var}(u(x,Y^*,\xi))$, $(G_\pi)^2 := \E_{x\sim \pi}\|\nabla_{\W}\sF(x,\pi)\|^2$ and $J_t := \sqrt{\E \| \nabla_\W \bar{\sE}(X_t,\rho_t(\cdot |Y,\xi)) \|^2}  $
\begin{enumerate}
    \item Suppose the function $x \to u(x,y,\xi)$ and $y \to u(x,y,\xi)$ are $L_u$-Lipschitz. Then, 
    \small
    \begin{equation}\label{Eq:Lemma2BoundDE_1}
    \begin{split}
    &\E[\de_1(t)] \leq L_u  J_t\cdot \sqrt{2t^2 \left( 2L_u^2 \W_2^2(\rho_0, \pi) + (\sigma^\ast)^2 + (G_\pi)^2 \right) + \sigma^2 td} \,.
    \end{split}
    \end{equation}
    \normalsize
    \item Suppose $\|\nabla_\W \sF(x,\mu)-\nabla_\W\sF(x,\nu)\| \leq L_\sF \W_2(\mu,\nu)$
    Then, 
    \small
    \begin{equation} \label{Eq:Lemma2BoundDE_2}
    \begin{split}
    &\E[\de_2(t)] \leq L_\sF  J_t\cdot \sqrt{ 2t^2 \left(2L_u^2 \W_2^2(\rho_0, \pi) + (\sigma^\ast)^2 + (G_\pi)^2\right) + \sigma^2 td } \,.
    \end{split}
    \end{equation}
    \normalsize
    \item Assume that $x \to u(x,y,\xi)$, $y\to u(x,y,\xi)$ and $x \to \nabla_{\W}\sF(x,\mu)$ are continuously differentiable and $L_u$ Lipschitz. Define $\Theta(x,y,\xi) := u(x, y, \xi) + \nabla_\W \sF(x,\rho_0)$. Assume $\|\nabla_\W \bar{\sF}(x_1,\mu) - \nabla_\W \bar{\sF}(x_2,\nu)\| \leq L_{\bar{u}}\|x_1-x_2\| + L_{\bar{\mathcal{F}}}\W_2(\mu,\nu)$
    \small
    \begin{equation} \label{Eq:Lemma2BoundSE}
    \begin{split}
    &\E[\se(t)] \leq  \sigma\sqrt{2t} L_u d \sqrt{\E_{x\sim \rho_0} [\Var(u(x,Y,\xi)]} \\&\quad +  2\sqrt{\E [\W_2^2(\rho_t, \rho_t(\cdot | Y,\xi))]} \left[(L_{\bar{u}} + L_{\bar{\sF}}) \sqrt{\E \big\| \Theta(X_t,Y,\xi) \big\|^2} + L_u \sqrt{\E\big\| \nabla_\W \bar{\sF}(X_t, \rho_t) \big\|^2}\right]
    \end{split}
    \end{equation}
    \normalsize
    \item Suppose $\|\nabla_\W \bar{\sF}(x,\mu)-\nabla_\W \sF(x,\mu)\| \leq L_l\W_2(\pi,\mu)$. Then,
    \small
    \begin{equation} \label{Eq:Lemma2BoundLE}
    \begin{split}
    \E[\le(t)] &\leq L_l J_t \cdot \sqrt{\E \left[\W_2^2(\pi, \rho_t(\cdot | Y,\xi))\right]} \,.
    \end{split}
    \end{equation}
    \normalsize
\end{enumerate}    
\end{lemma}

The Assumption~\ref{Assumption:Lipschitz} concerns the regularity of the functional $\bar{\sE}$ and its stochastic gradients, whereas Assumption~\ref{Assumption:functional} gives growth conditions and functional inequalities $\bar{\sE}$. Assumption~\ref{Assumption:BoundVariance} bounds the fluctuations of the stochastic gradient. As we show in specific examples of pairwise interacting systems (Section~\ref{subsec:pairwise}) and mean field neural networks (Section~\ref{subsec:mean_field}), these are implied by standard assumptions in the literature, and allow us to establish state-of-the-art convergence bounds.

\begin{assumption}[Lipschitz continuity]
\label{Assumption:Lipschitz}
For some $L_u,L_{\bar{u}},L_{\bar{\sF}}, L_{\sF} > 0$, the function $ x \to u(x,y,\xi)$ and $y \to u(x,y,\xi)$ are $L_u$-Lipschitz. For every $x,y\in \mR^d$, $\mu,\nu \in \sP_2(\R^d)$:
\begin{align*}    
   &(i)~~ \|\nabla_\W \bar{\sF}(x,\mu) - \nabla_\W \bar{\sF}(y,\nu)\| \leq L_{\bar{\sF}} \W_2(\mu,\nu) + L_{\bar{u}}\|x-y\| \\
    &(ii)~~\|\nabla_\W \sF(x,\mu) - \nabla_\W \sF(x,\nu)\| \leq L_{\sF} \W_2(\mu,\nu)  + L_u\|x-y\| \\
    & (iii) ~~\|\nabla_\W \bar{\sF}(x,\mu)-\nabla_\W \sF(x,\mu)\| \leq L_l\W_2(\pi,\mu)
\end{align*}
\end{assumption}

\begin{assumption}
\label{Assumption:functional}
Let $\pi$ be the minimizer of the functional $\bar{\sF} + \frac{\sigma^2}{2} \sH$. For some $C_{\bar{\sE}},C_{\mathsf{LSI}},C_{\KL} > 0$, the functional $\bar{\sE}$ satisfies the:
\small
\begin{align}
        &(i)~~ \|\nabla_\W \bar{\sE}(x,\mu)\|^2_{L^2(\mu)} \geq C_{\bar{\sE}}\bar{\sE}(\mu) \quad \text{for all } \mu\in \sP_{2,\mathsf{ac}}(\mR^d)  \tag{Polyak-{\L}ojasiewicz inequality} \label{eq:PL} \\
        &(ii)~ ~\KL(\mu \,||\, \pi) \leq \frac{C_{\mathsf{LSI}}}{2}\FD(\mu \,||\, \pi) \quad \text{for all } \mu\in \sP_{2,\mathsf{ac}}(\mR^d) \tag{Log-Sobolev inequality}\label{eq:LSI} \\
        &(iii)~ ~\KL(\mu || \pi) \leq C_{\KL}\bar{\sE}(\mu) \quad \text{for all }\mu\in \sP_{2,\mathsf{ac}}(\mR^d)  \tag{KL-Growth} \label{eq:kl_growth}
\end{align}
\normalsize
 
\end{assumption}

\begin{remark}
    We believe we can consider the more natural condition of $\W_2^2$-Growth- $\W_2^{2}(\mu,\pi) \leq C_{\W}\bar{\sE}(\mu)$-instead of KL-Growth by a straightforward modification of our proofs. 
\end{remark}
\begin{assumption}
\label{Assumption:BoundVariance}
If $Y,\xi\sim \rho_0\times \nu$, then from some $C^{\Var}, C_{\nu}^{\Var} > 0$, and for all $x\in \mR^d$: 
\[
    \E \| u(x,Y,\xi) + \nabla_\W \sF(x;\rho_0) \|^2 \leq C^{\Var} \Var(\rho_0) + C_{\nu}^{\Var} \,.
\]    
\end{assumption}

We now demonstrate the following descent lemma, proved in Section~\ref{subsec:descent_lemma_proof}.
\begin{lemma} [Descent lemma]\label{lem:descent_lemma}
Suppose that Assumptions \ref{Assumption:Lipschitz}, \ref{Assumption:functional} and ~\ref{Assumption:BoundVariance} are satisfied.
\begin{equation*}
\begin{split}
(\sigma^*)^2 := \E_{x\sim \pi}\mathsf{Var}(u(x,Y^*,\xi)) \,,\,  (G_\pi)^2 := \E_{x\sim \pi}\|\nabla_{\W}\sF(x,\pi)\|^2 \,,\, \\(G_{\mathsf{mod}})^2 := \E_{x\sim \pi}\|\nabla_{\W}\bar{\sF}(x,\pi)\|^2 \,.
\end{split}
\end{equation*}

Assume that the following inequalities are satisfied for some small enough $c_0 > 0$:
\begin{enumerate}
    \item $C_{\bar{\sE}} \geq 8L_l^2 C_{\KL}C_{\mathsf{LSI}} $ 
     \item $\eta < c_0 \min\left(\frac{1}{L_u},\frac{1}{C_{\bar{\sE}}},\frac{C_{\bar{\sE}}}{C_{\mathsf{LSI}}C_{\KL}L_u^2(L_{\bar{u}}+L_{\bar{\sF}})},\frac{1}{L_u(L_{\bar{u}}+L_{\bar{\sF}})}\sqrt{\frac{C_{\bar{\sE}}}{C_{\mathsf{LSI}}C_{\KL}}},\frac{C_{\bar{\sE}}}{C^{\Var}C_{\KL}C_{\mathsf{LSI}}(L_{\bar{u}}+L_{\bar{\sF}})}\right)$
\end{enumerate}

Then, for some universal constant $C > 0$:
\begin{align*}
    \E\bar{\sE}(\rho_{\eta}) &\leq e^{-\frac{\eta C_{\bar{\sE}}}{8}} \E \bar{\sE}(\rho_0) + C\left[\gamma_3 \eta^3 + \gamma_{2}\eta^2 + \gamma_{1}\eta^{\frac{3}{2}}\right] \,,
\end{align*}
where $\gamma_3 := (L_u^2+L^2_{\sF})((\sigma^*)^2 + (G_{\pi})^2) + \frac{L_u^2 G_{\mathsf{mod}}^2C^{\Var}C_{\mathsf{LSI}}C_{\KL}}{C_{\bar{\sE}}}$, $\gamma_2 := (L_u^2+L_{\sF}^2)\sigma^2 d + L_u G_{\mathsf{mod}}\sqrt{C^{\Var}\Var(\pi)+C^{\Var}_{\nu}} + (L_u+L_{\bar{u}}+L_{\bar{\sF}})(C^{\Var}\Var(\pi)+C_{\nu}^{\Var}) + \frac{\sigma^2 L_u^2 d^2 C^{\Var}C_{\mathsf{LSI}}C_{\KL}}{C_{\bar{\sE}}}$ and $\gamma_1:= \sigma d L_u \sqrt{C^{\Var}\Var(\pi)+C^{\Var}_{\nu}} \,.$

\end{lemma}
Unrolling the recursion established in Lemma~\ref{lem:descent_lemma} allows us to prove our main result:
\begin{theorem}\label{thm:main_theorem}
Consider Algorithm~\ref{alg:main_algo}. Let $\pi$ be the unique minimizer of $\bar{\sE}$. Suppose Assumptions~\ref{Assumption:Lipschitz},~\ref{Assumption:functional},~\ref{Assumption:BoundVariance} hold with $G_{\pi} = u$. With $\gamma_1,\gamma_2,\gamma_3$ as defined in Lemma~\ref{lem:descent_lemma}, the following holds:
\begin{enumerate}
    \item Conditioned on $\sR_{T-1}$, $X_T^{(1)},\dots,X_T^{(n)} \iidsim \mu_T|\sR_{T-1}$.
    \item $\E \bar{\sE}(\mu_T|\sR_{T-1}) \leq e^{-\frac{\eta C_{\bar{\sE}}T }{8}}\bar{\sE}(\mu_0) + C\left[\frac{\gamma_3\eta^2}{C_{\bar{\sE}}} + \frac{\gamma_2\eta}{C_{\bar{\sE}}} + \frac{\gamma_1\sqrt{\eta}}{C_{\bar{\sE}}}\right]$.
\end{enumerate}
\end{theorem}

\begin{remark}
Note that in Algorithm \ref{alg:main_algo}, if we use a batch size $B > 1$, then we would modify the gradient estimator as
\small
\[
    \hat{G}_B(x, Y_1, Y_2, \ldots, Y_B, \xi) := \frac{1}{B} \sum_{i = 1}^B \hat{G}(x, Y_i, \xi) \,,
\]
\normalsize
where $\{Y_i\}_{i = 1}^B$ are i.i.d. Therefore
\small
\[
    \Var(\hat{G}_B(x, Y_1, Y_2, \ldots, Y_B, \xi)) = \frac{1}{B}\Var(\hat{G}(x, Y_1, \xi)) \,.
\]
\normalsize
The parameter $B$ affects $\E[\se(t)]$ via $\Var(\hat{G}(x, Y, \xi))$. From Lemma \ref{lem:errorbounds}, we obtain that the first term (which is the leading order term) in \eqref{Eq:Lemma2BoundSE} changes to
\begin{equation} \label{Eq:BatchSizePrelim1}
    \frac{\sigma\sqrt{2t} L_u d}{\sqrt{B}} \sqrt{\E_{x\sim \rho_0} [\Var(u(x,Y_1,\xi)]} \,.
\end{equation}
\end{remark}

\subsection{Pairwise Interaction Energy Functional}
\label{subsec:pairwise}
Let $V,W: \mR^d \to \mR$ and $\mu\in \sP_{2, \mathsf{ac}}(\mR^d)$. Let $W$ be even (i.e, $W(x) = W(-x)$). Recall that the definition of the functional $\sF$, its Wasserstein gradient, and its unique minimizer are in \eqref{eq:int_energy}, \eqref{Eq:IEWassGradient}, \eqref{Eq:IEstationarymeasure} respectively. We call $f: \R^d \to \R$ to be $L$-smooth if  $\|\nabla f(x) - \nabla f(y)\| \leq L\|x-y\|$ for every $x,y \in \R^d$. 
\begin{assumption} [Smoothness]\label{assumption:pairwise_smoothness}
$V$ is $L_V$ smooth and $W$ is $L_W$ smooth. 
\end{assumption}

\begin{assumption} [LSI]\label{assumption:pairwise_lsi}
    $\pi$ satisfies LSI with constant $C_{\mathsf{LSI}}$, i.e., for all $\mu\in \sP_{2,\mathsf{ac}}(\mR^d)$:
    \small
    \[
        \KL(\mu \,||\, \pi) \leq \frac{C_{\mathsf{LSI}}}{2}\FD(\mu \,||\, \pi) \,.
    \]
    \normalsize
\end{assumption}
The assumption $L_W \leq \frac{\sigma^2}{\sqrt{24} C_{\mathsf{LSI}}}$ is called `` weak interaction'' in  \cite{kook2024sampling}. Our assumption below is less restrictive in terms of multiplicative constants.
\begin{assumption}[Weak Interaction]\label{assumption:pairwise_linearize}
$L_W \leq \frac{\sigma^2}{4C_{\mathsf{LSI}}}$.
\end{assumption}

We define the velocity field $\hat{G}(x,y) := -\nabla V(x) - \nabla W(x - y) \,,\, \forall \,\, x\in \mR^d $ and $\bar{\sE}(\mu) = \frac{\sigma^2}{2}\KL(\mu||\pi)$ (which corresponds to picking $\bar{\sF}(\mu) = \int V(x)d\mu(x) + \int W(x-y)d\pi(y)d\mu(x)$), where $\pi$ is the minimizer of Equation~\eqref{eq:int_energy}. The following Lemma establishes the general Assumptions required for Theorem~\ref{thm:main_theorem} using the Assumptions~\ref{assumption:pairwise_smoothness} and~\ref{assumption:pairwise_lsi}. We refer to Section~\ref{subsec:assumption_pairwise_proof} for the proof.
\begin{lemma}\label{lem:assumption_pairwise}
    Under Assumptions~\ref{assumption:pairwise_smoothness} and~\ref{assumption:pairwise_lsi}, the general Assumption~\ref{Assumption:Lipschitz} is satisfied with $L_u = L_{\bar{u}}= L_V+L_W$, $L_{\sF} = L_W$, $L_{\bar{\sF}} = 0$ and $L_l = L_{W}$.
The Assumption~\ref{Assumption:functional} is satisfied with $C_{\mathsf{LSI}} = C_{\mathsf{LSI}}$, $C_{\bar{\sE}} = \frac{\sigma^2}{C_{\mathsf{LSI}}}$ and $C_{\KL} = \frac{2}{\sigma^2}$. The Assumption~\ref{Assumption:BoundVariance} is satisfied with $C^{\Var} = 2L_W^2$ and $C^{\Var}_{\nu} = 0$.
\end{lemma}

Theorem~\ref{thm:main_pairwise}, proved in Section~\ref{subsec:main_pairwise_proof}, instantiates Theorem~\ref{thm:main_theorem} to the Pairwise Interaction Energy.


\begin{theorem}\label{thm:main_pairwise}
Consider the Pairwise Interaction Energy in Equation~\eqref{eq:int_energy}   under Assumptions~\ref{assumption:pairwise_smoothness},~\ref{assumption:pairwise_lsi} and~\ref{assumption:pairwise_linearize}. There exist universal constants $c_0,C > 0$ such that Algorithm~\ref{alg:main_algo} with $\hat{G}$ as above with
$\eta < c_0 \min\left(\tfrac{C_{\mathsf{LSI}}} {\sigma^2},\tfrac{\sigma^4}{C_{\mathsf{LSI}}^2(L_V+L_W)^3}\right)$ satisfies:

\small
\begin{align} 
\E\left[ \KL(\mu_T|\sR_{T-1}||\pi) \right] &\leq  e^{\big(-\tfrac{T \sigma^2 \eta}{8C_{\mathsf{LSI}}}\big)} \KL(\mu_0||\pi)  + C\frac{\sqrt{\eta} d^{3/2} (C_{\mathsf{LSI}})^{1/2} \sigma(L_V + L_W)}{4} \label{Eq:MainTheoremPairwise}\\
&+ C \eta (L_V + L_W)^2 d^2 C_{\mathsf{LSI}} \nonumber \,. \end{align}
\end{theorem}
\normalsize

\begin{remark}
Given $\epsilon/3 \in (0,\KL(\mu_0||\pi)\wedge 1)$, as per Theorem~\ref{thm:main_pairwise}, we can achieve $\E\left[\KL(\mu_T|\sR_{T-1}) - \KL(\pi)\right] \leq \epsilon$ by picking:
\begin{enumerate}
\item 
$\eta = \frac{8C_{\mathsf{LSI}}}{\sigma^2 T}\log \left(\tfrac{3\KL(\mu_0 || \pi)}{\epsilon} \right)$ 
\item $T \gtrsim \max\left(\frac{C_{\mathsf{LSI}}^2 d^3(L_V+L_W)^2}{\epsilon^2},\frac{C_{\mathsf{LSI}}^2 d^2 (L_V+L_W)^2}{\sigma^2\epsilon},\frac{C_{\mathsf{LSI}}^3(L_V+L_W)^3}{\sigma^6}\right)\log \left(\tfrac{3\KL(\mu_0||\pi)}{\epsilon} \right)$
\end{enumerate}

\end{remark}

\paragraph{Comparison with prior work}
We refer to the general discussion on establishing isoperimetry for $n$ particle systems in Section \ref{sec:intro}. In the specific case of pairwise interaction, prior work \cite{kook2024sampling} considers the standard particle algorithm for this problem and shows that the $n$-particle stationary distribution with $n = \frac{\sqrt{d}}{\epsilon} $, which ensures that the law of the first particle $Y_1$ is such that $\mathsf{KL}(\Law(Y_1)||\pi) \leq \epsilon^2$. To obtain sampling guarantees for the $n$ particle system, they further assume the potentials $V, W$ are decomposed as $V = V_0 + V_1, W = W_0 + W_1$, where $V_0, W_0$ are uniformly convex and $\mathsf{osc}(V_1) := \sup V_1 - \inf V_1 < \infty, \mathsf{osc}(W_1) := \sup W_1 - \inf W_1 < \infty$. Note that $V_1, W_1$ can be interpreted as perturbations of $V_0, W_0$, which are uniformly convex, and the assumption $\mathsf{osc}(V_1), \mathsf{osc}(W_1) < \infty$ means the perturbations $V_1, W_1$ are bounded. The isoperimetric constant $C_{\mathsf{LSI},n}$ is then bounded with an exponential dependence on $\mathsf{osc}(V_1) + \mathsf{osc}(W_1)$. Our work does not require these additional assumptions. However, the standard particle based method as above can use any standard sampling algorithm to sample from the $n$-particle stationary distribution (not restricted to LMC style Euler-Maruyama discretization as in this work).

\subsection{Mean Field Neural Network}
\label{subsec:mean_field}

In~\eqref{eq:mean_field}, we assume $P$ to be empirical distribution of $m$ data samples $(z_1,w_1),\dots,(z_m,w_m)$. Then, the functional in \eqref{eq:mean_field} simplifies to:
\begin{equation} \label{Eq:MFFunctional2}
    \sF(\mu) = \frac{1}{m}\sum_{i=1}^{m}\left(\int h(z_i,x)d\mu(x) - w_i \right)^2 + \frac{\lambda}{2} \int \|x\|^2d\mu(x) \,,
\end{equation}
 The Wasserstein gradient of this functional follows from \eqref{Eq:MFNWassGradient} to be: 
\[
\nabla_{\W}\sF(x;\mu) = \frac{2}{m}\sum_{i=1}^{m}\left(\int h(z_i,y)d\mu(y) - w_i \right)\nabla_x h(z_i,x) + \lambda x \,.
\]
The unique minimizer of the functional $\sE$ is given in \eqref{eq:mf_optimal}. We set $\bar{\sE}(\mu) = \sE(\mu) - \sE(\pi)$. We consider the proximal Gibbs distribution corresponding to $\mu$, which is given by:
\[
    \pi_{\mu}(x) \propto \exp\left(-\frac{2\delta \sF(x,\mu)}{\sigma^2}\right) \,,
\]
where $\delta \sF$ is the first-variation of the functional $\sF$. We let $\nu = \mathsf{Unif}([m])$. We denote the random variable $\xi$, used in the definition of the Wasserstein gradient estimator $\hat{G}$, by $I$ and choose $\hat{G}$ to be:
\[
    \hat{G}(x,y,i) := -(h(z_i,y) - w_i)\nabla_x h(z_i,x) - \lambda x \,,\, \forall \,\, x,y\in \mR^d \,, i \in [m] \,.
\]

\begin{assumption} [Boundedness]
\label{Assumption:MFNBoundedness}
    For every $x,z\in \mR^d$:
    $$\|h(z,x)\| \leq B ;\,\quad\, \|z\|\leq R; \quad  |w| \leq R; \quad \|\nabla_x h(z,x)\| \leq M\|z\| \,.$$
\end{assumption}

\begin{assumption} [Lipschitz continuity]
\label{Assumption:MFNLipschitz}
    The function $x \to \nabla h(z,x)$ is $L \|z\|$-Lipschitz.
\end{assumption}

Excluding special cases, the boundedness assumption on $h,\nabla_x h$ and Lipschitz assumption on $\nabla_x h$ are necessary to satisfy the general assumptions in prior works \cite{wang2024uniform,nitanda2023improved} for the square loss. For instance, we refer to \cite[Assumption 1]{nitanda2023improved}. 

\begin{assumption} [LSI]
\label{Assumption:MFNLSI}
    For any $q \in \sP_2(\mR^d)$, the proximal Gibbs distribution $\pi_q$ satisfies the LSI with constant $C_{\mathsf{LSI}}$. $\pi$ also satisfies LSI with the same constant.
\end{assumption}

The following lemma is proved in Section~\ref{subsec:assumption_mean_field_proof}
\begin{lemma}\label{lem:assumption_mean_field}
    Under Assumptions~\ref{Assumption:MFNBoundedness},~\ref{Assumption:MFNLipschitz} and~\ref{Assumption:MFNLSI}, the general Assumption~\ref{Assumption:Lipschitz} is satisfied with $L_u = L_{\bar{u}}= (B + R)LR + \lambda + M^2R^2$, $L_{\sF} = L_{\bar{\sF}} = M^2R^2 $, and $L_l = 0$.
The general Assumption~\ref{Assumption:functional} is satisfied with $C_{\mathsf{LSI}} = C_{\mathsf{LSI}}$, $C_{\bar{\sE}} = \tfrac{\sigma^2}{C_{\mathsf{LSI}}}$ and $C_{\KL} = \frac{2}{\sigma^2}$. The Assumption~\ref{Assumption:BoundVariance} is satisfied with $C^{\Var} = 0$ and $C^{\Var}_\nu = 4M^2R^2(B+R)^2$.
\end{lemma}

\begin{theorem}\label{thm:main_mean_field}
Consider the case of the Mean Field Neural Network with square loss in Equation~\eqref{eq:mean_field} under Assumptions~\ref{Assumption:MFNLipschitz},~\ref{Assumption:MFNBoundedness} and~\ref{Assumption:MFNLSI}. We consider Algorithm~\ref{alg:main_algo} with $\hat{G}$ as defined above and $\eta < c_0 \min\left(\frac{C_{\mathsf{LSI}}}{\sigma^2},\frac{\sigma^4}{C^2_{\mathsf{LSI}}L_u^3}\right)$ for some $c_0 > 0$ small enough and $L_u = (B + R)LR + \lambda + M^2R^2$. Then for some universal constant $C > 0$: \small
\begin{align*}
 &\E\sE(\mu_T|\sR_{T-1})-\sE(\pi)\nonumber \\ &\leq e^{-\frac{T\eta \sigma^2}{8C_{\mathsf{LSI}}}}(\sE(\mu_0) - \sE(\pi)) + C\frac{C_{\mathsf{LSI}}}{\sigma^2}\left[ \eta (\sigma^2 L^2_u d + L_uM^2R^2(B+R)^2)  + \sqrt{\eta}\sigma d L_u M R (B+R) \right] \,.  
\end{align*} \normalsize
\end{theorem}
\begin{remark}
Given $\epsilon/3 \in (0,\sE(\mu_0)-\sE(\pi))$, as per Theorem~\ref{thm:main_mean_field}, we can achieve $\E\sE(\mu_T|\sR_{T-1}) - \sE(\pi) \leq \epsilon$ by picking
\begin{enumerate}
\item 
$\eta = \frac{8C_{\mathsf{LSI}}}{\sigma^2  T}\log\left(\tfrac{3(\sE(\mu_0)-\sE(\pi))}{\epsilon}\right)$ 
\item $T \gtrsim \max\left(\frac{C_{\mathsf{LSI}}^3d^2L_u^2M^2R^2(B+R)^2}{\sigma^4\epsilon^2},\frac{C_{\mathsf{LSI}}^2(\sigma^2L_u^2d +L_uM^2R^2(B+R)^2)}{\sigma^4\epsilon},\frac{L_u^3 C_{\mathsf{LSI}}^3}{\sigma^6}\right)\log\left(\tfrac{3(\sE(\mu_0)-\sE(\pi))}{\epsilon}\right)$
\end{enumerate}

\end{remark}

\section{Conclusion}
We introduce a novel convergence analysis of virtual particle stochastic approximation, achieving state-of-the-art convergence when unbiased estimators for the Wasserstein gradient are accessible. Future research directions include extending this analysis to scenarios involving biased-but-consistent estimators, leveraging techniques to convert them into unbiased estimators with heavy tails \cite{blanchet2015unbiased}.  Furthermore, investigating gradient flows within the Wasserstein-Fischer-Rao Geometry \cite{chizat2018unbalanced,chizat2018interpolating}, which has demonstrated significant potential in statistics and inference, presents a promising avenue. Finally, exploring the algorithm's behavior under weaker conditions, such as the Poincaré inequality, is another interesting area for future study.

\section{Acknowledgements}
Anant Raj was supported by a grant from Ittiam Systems Private Limited through the Ittiam Equitable AI Lab.
\bibliography{references}

\newcommand{\etalchar}[1]{$^{#1}$}
\begin{thebibliography}{GBWW21}

\bibitem[AGS08]{ambrosio2008metric}
Luigi Ambrosio, Nicola Gigli, and Giuseppe Savaré.
\newblock {\em Gradient Flows: In Metric Spaces and in the Space of Probability Measures}.
\newblock Springer Science \& Business Media, 2008.

\bibitem[AKL21]{amir2021sgd}
Idan Amir, Tomer Koren, and Roi Livni.
\newblock Sgd generalizes better than gd (and regularization doesn’t help).
\newblock In {\em Conference on Learning Theory}, pages 63--92. PMLR, 2021.

\bibitem[AS07]{ambrosiosavarenotes}
Luigi Ambrosio and Giuseppe Savar{\'e}.
\newblock Gradient flows of probability measures.
\newblock In {\em Handbook of differential equations: evolutionary equations}, volume~3, pages 1--136. Elsevier, 2007.

\bibitem[BBG24]{balasubramanian2024improved}
Krishnakumar Balasubramanian, Sayan Banerjee, and Promit Ghosal.
\newblock Improved finite-particle convergence rates for stein variational gradient descent.
\newblock {\em arXiv preprint arXiv:2409.08469}, 2024.

\bibitem[BG15]{blanchet2015unbiased}
Jose~H Blanchet and Peter~W Glynn.
\newblock Unbiased monte carlo for optimization and functions of expectations via multi-level randomization.
\newblock In {\em 2015 Winter Simulation Conference (WSC)}, pages 3656--3667. IEEE, 2015.

\bibitem[BGL13]{bakry2013analysis}
Dominique Bakry, Ivan Gentil, and Michel Ledoux.
\newblock {\em Analysis and geometry of Markov diffusion operators}, volume 348.
\newblock Springer Science \& Business Media, 2013.

\bibitem[BRS23]{bou2023nonlinear}
Nawaf Bou-Rabee and Katharina Schuh.
\newblock Nonlinear hamiltonian monte carlo \& its particle approximation.
\newblock {\em arXiv preprint arXiv:2308.11491}, 2023.

\bibitem[CB18]{chizat2018global}
Lenaic Chizat and Francis Bach.
\newblock On the global convergence of gradient descent for over-parameterized models using optimal transport.
\newblock {\em Advances in neural information processing systems}, 31, 2018.

\bibitem[CD15]{carmona2015forward}
Ren{\'e} Carmona and Fran{\c{c}}ois Delarue.
\newblock Forward--backward stochastic differential equations and controlled mckean--vlasov dynamics.
\newblock {\em The Annals of Probability}, pages 2647--2700, 2015.

\bibitem[Chi22]{chizat2022mean}
L{\'e}na{\"\i}c Chizat.
\newblock Mean-field langevin dynamics: Exponential convergence and annealing.
\newblock {\em arXiv preprint arXiv:2202.01009}, 2022.

\bibitem[CMSS20]{chen2020finite}
Zaiwei Chen, Siva~Theja Maguluri, Sanjay Shakkottai, and Karthikeyan Shanmugam.
\newblock Finite-sample analysis of contractive stochastic approximation using smooth convex envelopes.
\newblock {\em Advances in Neural Information Processing Systems}, 33:8223--8234, 2020.

\bibitem[CNZ24]{chewi2024uniform}
Sinho Chewi, Atsushi Nitanda, and Matthew~S Zhang.
\newblock Uniform-in-$ n $ log-sobolev inequality for the mean-field langevin dynamics with convex energy.
\newblock {\em arXiv preprint arXiv:2409.10440}, 2024.

\bibitem[CPSV18a]{chizat2018interpolating}
Lenaic Chizat, Gabriel Peyr{\'e}, Bernhard Schmitzer, and Fran{\c{c}}ois-Xavier Vialard.
\newblock An interpolating distance between optimal transport and fisher--rao metrics.
\newblock {\em Foundations of Computational Mathematics}, 18:1--44, 2018.

\bibitem[CPSV18b]{chizat2018unbalanced}
Lenaic Chizat, Gabriel Peyr{\'e}, Bernhard Schmitzer, and Fran{\c{c}}ois-Xavier Vialard.
\newblock Unbalanced optimal transport: Dynamic and kantorovich formulations.
\newblock {\em Journal of Functional Analysis}, 274(11):3090--3123, 2018.

\bibitem[CRW22]{chen2022uniform}
Fan Chen, Zhenjie Ren, and Songbo Wang.
\newblock Uniform-in-time propagation of chaos for mean field langevin dynamics.
\newblock {\em arXiv preprint arXiv:2212.03050}, 2022.

\bibitem[DMM19]{durmus2019analysis}
Alain Durmus, Szymon Majewski, and B{\l}a{\.z}ej Miasojedow.
\newblock Analysis of langevin monte carlo via convex optimization.
\newblock {\em Journal of Machine Learning Research}, 20(73):1--46, 2019.

\bibitem[DN23]{dasnagaraj2023}
Aniket Das and Dheeraj Nagaraj.
\newblock Provably fast finite particle variants of svgd via virtual particle stochastic approximation.
\newblock {\em Advances in Neural Information Processing Systems}, 36:49748--49760, 2023.

\bibitem[DNR23]{dasnagarajraj2023}
Aniket Das, Dheeraj~M Nagaraj, and Anant Raj.
\newblock Utilising the clt structure in stochastic gradient based sampling: Improved analysis and faster algorithms.
\newblock In {\em The Thirty Sixth Annual Conference on Learning Theory}, pages 4072--4129. PMLR, 2023.

\bibitem[Duc07]{duchi2007derivations}
John Duchi.
\newblock Derivations for linear algebra and optimization.
\newblock {\em Berkeley, California}, 3(1):2325--5870, 2007.

\bibitem[GBWW21]{godichon2021non}
Antoine Godichon-Baggioni, Nicklas Werge, and Olivier Wintenberger.
\newblock Non-asymptotic analysis of stochastic approximation algorithms for streaming data.
\newblock {\em arXiv preprint arXiv:2109.07117}, 2021.

\bibitem[HZM{\etalchar{+}}24]{huang2024faster}
Xunpeng Huang, Difan Zou, Yi-An Ma, Hanze Dong, and Tong Zhang.
\newblock Faster sampling via stochastic gradient proximal sampler.
\newblock {\em arXiv preprint arXiv:2405.16734}, 2024.

\bibitem[JKO98]{jordan1998variational}
Richard Jordan, David Kinderlehrer, and Felix Otto.
\newblock The variational formulation of the fokker--planck equation.
\newblock {\em SIAM journal on mathematical analysis}, 29(1):1--17, 1998.

\bibitem[JLL20]{jin2020random}
Shi Jin, Lei Li, and Jian-Guo Liu.
\newblock Random batch methods (rbm) for interacting particle systems.
\newblock {\em Journal of Computational Physics}, 400, 2020.

\bibitem[KN24]{kandasamy2024poisson}
Saravanan Kandasamy and Dheeraj Nagaraj.
\newblock The poisson midpoint method for langevin dynamics: Provably efficient discretization for diffusion models.
\newblock {\em arXiv preprint arXiv:2405.17068}, 2024.

\bibitem[KS22]{kinoshita2022improved}
Yuri Kinoshita and Taiji Suzuki.
\newblock Improved convergence rate of stochastic gradient langevin dynamics with variance reduction and its application to optimization.
\newblock {\em Advances in Neural Information Processing Systems}, 35:19022--19034, 2022.

\bibitem[KY03]{kushner2003stochastic}
Harold Kushner and G~George Yin.
\newblock {\em Stochastic approximation and recursive algorithms and applications}, volume~35.
\newblock Springer Science \& Business Media, 2003.

\bibitem[KZC{\etalchar{+}}24]{kook2024sampling}
Yunbum Kook, Matthew~S Zhang, Sinho Chewi, Murat~A Erdogdu, and Mufan~Bill Li.
\newblock Sampling from the mean-field stationary distribution.
\newblock In {\em The Thirty Seventh Annual Conference on Learning Theory}, pages 3099--3136. PMLR, 2024.

\bibitem[Lac23]{lacker2023independent}
Daniel Lacker.
\newblock Independent projections of diffusions: gradient flows for variational inference and optimal mean field approximations.
\newblock {\em arXiv preprint arXiv:2309.13332}, 2023.

\bibitem[LCB{\etalchar{+}}22]{lambert2022variational}
Marc Lambert, Sinho Chewi, Francis Bach, Silv{\`e}re Bonnabel, and Philippe Rigollet.
\newblock Variational inference via wasserstein gradient flows.
\newblock {\em Advances in Neural Information Processing Systems}, 35:14434--14447, 2022.

\bibitem[LLF23]{lacker2023sharp}
Daniel Lacker and Luc Le~Flem.
\newblock Sharp uniform-in-time propagation of chaos.
\newblock {\em Probability Theory and Related Fields}, 187(1-2):443--480, 2023.

\bibitem[LW16]{liu2016stein}
Qiang Liu and Dilin Wang.
\newblock Stein variational gradient descent: A general purpose bayesian inference algorithm.
\newblock {\em Advances in neural information processing systems}, 29, 2016.

\bibitem[Mal01]{malrieu2001logarithmic}
Florient Malrieu.
\newblock Logarithmic sobolev inequalities for some nonlinear pde's.
\newblock {\em Stochastic processes and their applications}, 95(1):109--132, 2001.

\bibitem[Mal03]{malrieu2003convergence}
Florent Malrieu.
\newblock Convergence to equilibrium for granular media equations and their euler schemes.
\newblock {\em The Annals of Applied Probability}, 13(2):540--560, 2003.

\bibitem[McC97]{mccann1997convexity}
Robert~J McCann.
\newblock A convexity principle for interacting gases.
\newblock {\em Advances in mathematics}, 128(1):153--179, 1997.

\bibitem[MMN18]{mei2018mean}
Song Mei, Andrea Montanari, and Phan-Minh Nguyen.
\newblock A mean field view of the landscape of two-layer neural networks.
\newblock {\em Proceedings of the National Academy of Sciences}, 115(33):E7665--E7671, 2018.

\bibitem[Nit24]{nitanda2023improved}
Atsushi Nitanda.
\newblock Improved particle approximation error for mean field neural networks.
\newblock {\em arXiv preprint arXiv:2405.15767}, 2024.

\bibitem[NS17]{nitanda2017stochastic}
Atsushi Nitanda and Taiji Suzuki.
\newblock Stochastic particle gradient descent for infinite ensembles.
\newblock {\em arXiv preprint arXiv:1712.05438}, 2017.

\bibitem[NWS22]{nitanda2022convex}
Atsushi Nitanda, Denny Wu, and Taiji Suzuki.
\newblock Convex analysis of the mean field langevin dynamics.
\newblock In {\em International Conference on Artificial Intelligence and Statistics}, pages 9741--9757. PMLR, 2022.

\bibitem[OV00]{ottovillani}
Felix Otto and C{\'e}dric Villani.
\newblock Generalization of an inequality by talagrand and links with the logarithmic sobolev inequality.
\newblock {\em Journal of Functional Analysis}, 173(2):361--400, 2000.

\bibitem[PJ92]{polyak1992acceleration}
Boris~T Polyak and Anatoli~B Juditsky.
\newblock Acceleration of stochastic approximation by averaging.
\newblock {\em SIAM journal on control and optimization}, 30(4):838--855, 1992.

\bibitem[RM51]{robbins1951stochastic}
Herbert Robbins and Sutton Monro.
\newblock A stochastic approximation method.
\newblock {\em The annals of mathematical statistics}, pages 400--407, 1951.

\bibitem[RRT17]{raginsky2017non}
Maxim Raginsky, Alexander Rakhlin, and Matus Telgarsky.
\newblock Non-convex learning via stochastic gradient langevin dynamics: a nonasymptotic analysis.
\newblock In {\em Conference on Learning Theory}, pages 1674--1703. PMLR, 2017.

\bibitem[SL19]{shen2019randomized}
Ruoqi Shen and Yin~Tat Lee.
\newblock The randomized midpoint method for log-concave sampling.
\newblock {\em Advances in Neural Information Processing Systems}, 32, 2019.

\bibitem[SNW23]{suzuki2023uniform}
Taiji Suzuki, Atsushi Nitanda, and Denny Wu.
\newblock Uniform-in-time propagation of chaos for the mean-field gradient langevin dynamics.
\newblock In {\em The Eleventh International Conference on Learning Representations}, 2023.

\bibitem[SWN24]{suzuki2024mean}
Taiji Suzuki, Denny Wu, and Atsushi Nitanda.
\newblock Mean-field langevin dynamics: Time-space discretization, stochastic gradient, and variance reduction.
\newblock {\em Advances in Neural Information Processing Systems}, 36, 2024.

\bibitem[Szn91]{sznitman1991topics}
Alain-Sol Sznitman.
\newblock Topics in propagation of chaos.
\newblock {\em Ecole d’{\'e}t{\'e} de probabilit{\'e}s de Saint-Flour XIX—1989}, 1464:165--251, 1991.

\bibitem[Vil21]{villani2021OptTransport}
C{\'e}dric Villani.
\newblock {\em Topics in optimal transportation}, volume~58.
\newblock American Mathematical Soc., 2021.

\bibitem[VW19]{vempalawibisono}
Santosh Vempala and Andre Wibisono.
\newblock Rapid convergence of the unadjusted langevin algorithm: Isoperimetry suffices.
\newblock {\em Advances in neural information processing systems}, 32, 2019.

\bibitem[Wan24]{wang2024uniform}
Songbo Wang.
\newblock Uniform log-sobolev inequalities for mean field particles with flat-convex energy.
\newblock {\em arXiv preprint arXiv:2408.03283}, 2024.

\bibitem[WT11]{welling2011bayesian}
Max Welling and Yee~W Teh.
\newblock Bayesian learning via stochastic gradient langevin dynamics.
\newblock In {\em Proceedings of the 28th international conference on machine learning (ICML-11)}, pages 681--688. Citeseer, 2011.

\bibitem[YKD23]{yu2023langevin}
Lu~Yu, Avetik Karagulyan, and Arnak Dalalyan.
\newblock Langevin monte carlo for strongly log-concave distributions: Randomized midpoint revisited.
\newblock {\em arXiv preprint arXiv:2306.08494}, 2023.

\bibitem[YWR24]{yan2024learning}
Yuling Yan, Kaizheng Wang, and Philippe Rigollet.
\newblock Learning gaussian mixtures using the wasserstein--fisher--rao gradient flow.
\newblock {\em The Annals of Statistics}, 52(4):1774--1795, 2024.

\bibitem[YY22]{yao2022mean}
Rentian Yao and Yun Yang.
\newblock Mean field variational inference via wasserstein gradient flow.
\newblock {\em arXiv preprint arXiv:2207.08074}, 2022.

\end{thebibliography}
\bibliographystyle{alpha}

\newpage
\appendix


\section{Useful Technical Results}
Recall that the random variable $X_t$ is defined as $X_t := X_0 + tu(X_0,Y,\xi) + \sigma B_t$, where $u: \R^d\times\R^d \times \R^m \to \R^d$ is the velocity field. Also, note that $\rho_t := \Law(X_t)$ and $\rho_t(\cdot |Y,\xi) := \Law(X_t | Y,\xi)$. Let $\Var(\mu)$ denote the trace of the covariance matrix of the probability measure $\mu\in \sP_2(\mR^d)$ and $\Var(u(x,Y,\xi))$ the variance of the velocity field $u$ with respect to the random variables $Y,\xi$, for every fixed $x\in \mR^d$. Next, $\W_p(\mu,\nu)$ denotes the $p$-Wasserstein distance between the probability measures $\mu, \nu \in \sP_2(\mR^d)$ and $\nabla_{\W} \sF$ the Wasserstein gradient of the functional $\sF$. Recall that, for a functional $\bar{\sF}$, we define $\bar{\sE}(\mu) := \bar{\sF}(\mu) + \frac{\sigma^2}{2}\sH(\mu) - \bar{\sF}(\pi) - \frac{\sigma^2}{2}\sH(\pi)$, where $\pi$ is the minimizer of the functional $\bar{\sF} + \frac{\sigma^2}{2}\sH$. The functional $\bar{\sE}$ need not be the same as the functional $\sE$. Let $C_{\bar{\sE}}, C_{\mathsf{LSI}}, C_{\mathsf{KL}}$ be the constants corresponding to \eqref{eq:PL}, \eqref{eq:LSI}, \eqref{eq:kl_growth} in Assumption \ref{Assumption:functional}, respectively and $L_{\bar{u}}, L_{\bar{\sF}}$ the constants defined in Assumption~\ref{Assumption:Lipschitz}. Furthermore, $\| \cdot \|$ denotes the $L^2$ norm of a function.

The following lemma bounds the expected Wasserstein distance between $\rho_t,\rho_{t}(|Y,\xi)$.
\begin{lemma}\label{lem:cond_wass_bound}
    $$\E \W_2^2(\rho_t,\rho_t(\cdot |Y,\xi)) \leq  2t^2\E_{x\sim \rho_0}\mathsf{Var}(u(x,Y,\xi))$$

$$\E \W_2^2(\rho_0,\rho_t(\cdot |Y,\xi)) \leq \sigma^2 td + t^2\E \|u(X_0,Y,\xi)\|^2$$
\end{lemma}

\begin{proof}
Notice that $\rho_t(\cdot) = \E \rho_t(\cdot|Y,\xi)$. Since Wasserstein distance is convex in each of its co-ordinates and $x \to x^2$ is increasing and convex over $\R^{+}$, we have: 
$$\E\W_2^{2}(\rho_t,\rho_t(\cdot |Y,\xi)) \leq \E\W_2^{2}(\rho_t(\cdot |Y',\xi'),\rho_t(\cdot |Y,\xi))\,,$$
where $(Y',\xi')$ is an independent copy of $(Y,\xi)$. We now couple $\rho_t(\cdot |Y',\xi')$ and $\rho_t(\cdot |Y,\xi)$ for a given $Y,\xi$ as follows: 
$$X'_t = X_0 + tu(X_0,Y',\xi') + \sigma B_t \,;\quad\quad X_t = X_0 + tu(X_0,Y,\xi) + \sigma B_t \,.$$
Therefore, 
\begin{align}
    \E\W_2^{2}(\rho_t,\rho_t(\cdot |Y,\xi)) \leq t^2\E \|u(X_0,Y',\xi')-u(X_0,Y,\xi)\|^2 = 2t^2\E_{x\sim \rho_0}\mathsf{Var}(u(x,Y,\xi))
\end{align}
Finally, we couple $\rho_0$ and $\rho_t(\cdot |Y,\xi)$ by $X_0 \sim \rho_0$ and $X_t = X_0 + tu(X_0,Y,\xi) + \sigma B_t$. This implies,
$\E \W_2^{2}(\rho_0,\rho_t(\cdot |Y,\xi)) \leq \sigma^2 td + t^2\E \|u(X_0,Y,\xi)\|^2$.
\end{proof}

\begin{lemma}[Lemma 11 in \cite{vempalawibisono}]\label{lem:energy_gradient}
    Let $\nu$ be a probability measure over $\R^d$ with density $\nu(x) \propto e^{-F(x)}$ where $F$ is $L$-smooth. Then, 
    $$\E_{x\sim \nu}\|\nabla F(x)\|^2 \leq dL$$
\end{lemma}

\begin{lemma}[Otto-Villani Theorem \cite{ottovillani}]
\label{lem:OttoVillaniTheorem}
    Let $\pi$ satisfy LSI with constant $C_{\mathsf{LSI}}$. Then $\pi$ satisfies Talangrand's inequality $T_p$ for any $p\in [1,2]$, i.e., for all $\mu \in \sP_{2,ac}(\mR^d)$:
    \[
        \W^2_p(\mu, \pi) \leq 2C_{\mathsf{LSI}} \KL(\mu,\pi) \,.
    \]
\end{lemma}

\begin{lemma}\label{lem:control_from_pi}
For any $\mu,\pi \in \sP_2(\R^d)$, we have:

\begin{enumerate}
    \item     $$\mathsf{Var}(\mu) \leq 2 \W_2^{2}(\mu,\pi) + 2\mathsf{Var}(\pi)$$
    \item  Suppose Assumption~\ref{Assumption:Lipschitz} holds. Let $X \sim \mu$ and $X^{*} \sim \pi$. Then:
$$\E\|\nabla_{\W}\bar{\sF}(X,\mu)\|^2 \leq 3(L^2_{\bar{u}} + L^2_{\bar{\sF}})\W_2^{2}(\mu,\pi)+3\E\|\nabla_{\W}\bar{\sF}(X^*,\pi)\|^2 $$
\end{enumerate}
\end{lemma}
\begin{proof}
Let $X,X'$ be i.i.d. from $\mu$. Let $Y,Y'$ be i.i.d. from $\pi$ such that $Y,X$ are optimally coupled and $Y',X'$ are optimally coupled. Now consider:
\begin{align*}
    \Var(\mu) &= \frac{1}{2}\E \|X-Y + Y-Y' + Y'-X'\|^2 \nonumber \\
    &\leq \E\|X-Y + Y'-X'\|^2  + \E \|Y-Y'\|^2 \nonumber \\
    &= \E\|X-Y + Y'-X'\|^2  + 2\Var(\pi) \nonumber \\
    &= 2\E\|X-Y\|^2 - 2\|\E X-\E Y\|^2 + 2 \Var(\pi) \nonumber \\
    &\leq 2 \W_2^2(\mu,\pi) + 2\Var(\pi) \,,
\end{align*}
where in the penultimate step we have used the fact that $X-Y,X'-Y'$ are i.i.d. Next, let $X^\ast \sim \pi$ be optimally coupled to $X\sim \mu$. By Assumption \ref{Assumption:Lipschitz} and triangle inequality, we have:
\begin{align*}
    \|\nabla_{\W}\bar{\sF}(X,\mu)\| &= \|\nabla_{\W}\bar{\sF}(X,\mu) - \nabla_{\W}\bar{\sF}(X^\ast,\pi) + \nabla_{\W}\bar{\sF}(X^\ast,\pi)\| \\
    &\leq  L_{\bar{\sF}} \W_2(\mu, \pi) + L_{\bar{u}}\|X - X^\ast\| + \|\nabla_{\W}\bar{\sF}(X^\ast,\pi) \| \,.
\end{align*}
By squaring, applying $(a+b+c)^2 \leq 3(a^2 + b^2 + c^2)$, and taking expectation proves the second statement of this lemma.
\end{proof}

\begin{lemma}\label{lem:var_control}
$V(u,\rho_0) := \E_{x\sim \rho_0}[\Var(u(x,Y,\xi))]$. Under Assumptions~\ref{Assumption:functional} and~\ref{Assumption:BoundVariance}, we have:
\begin{align}
    V(u,\rho_0) 
    &\lesssim C^{\Var} C_{\mathsf{LSI}} C_{\KL} \bar{\sE}(\rho_0) + C^{\Var}\Var(\pi) + C^{\Var}_{\nu} \\
    \Var(\rho_0) &\lesssim C_{\mathsf{LSI}}C_{\KL}\bar{\sE}(\rho_0)+ \Var(\pi)
\end{align}
\end{lemma}
\begin{proof}
    By Assumption~\ref{Assumption:BoundVariance}, we have:

\begin{align}
    V(u,\rho_0) &\leq C^{\Var}\Var(\rho_0) + C^{\Var}_{\nu} \nonumber \\
    &\lesssim C^{\Var} \W_2^{2}(\rho_0,\pi) + C^{\Var}\Var(\pi) + C^{\Var}_{\nu}\tag{By Lemma~\ref{lem:control_from_pi}} \nonumber \\
        &\lesssim C^{\Var} C_{\mathsf{LSI}} \KL(\rho_0||\pi) + C^{\Var}\Var(\pi) + C^{\Var}_{\nu} \tag{By Lemma~\ref{lem:OttoVillaniTheorem}}\nonumber \\
    &\lesssim C^{\Var} C_{\mathsf{LSI}} C_{\KL} \bar{\sE}(\rho_0) + C^{\Var}\Var(\pi) + C^{\Var}_{\nu} \tag{By Assumption~\ref{Assumption:functional}-\eqref{eq:kl_growth}}
\end{align}
Applying a similar reasoning to the bounds in Lemma~\ref{lem:control_from_pi}, we conclude the bound on $\Var(\rho_0)$
\end{proof}

\begin{lemma}\label{lem:stationarity_bounds}
    Consider a probability measure $\pi$ over $\R^d$, with density $\pi(x) \propto \exp(-F(x))$, which satisfies the LSI with constant $C_{\mathsf{LSI}}$. Then, $\Var(\pi) \leq d C_{\mathsf{LSI}}$.
\end{lemma}

\begin{proof}
By \cite{bakry2013analysis}[Definition 4.2.1], the probability measure $\pi$ satisfies the Poincar\'e inequality with constant $C_{\mathsf{PI}} > 0$ if
    \[
        \Var_{\pi}(f) \leq C_{\mathsf{PI}} \E_\pi [\| \nabla f \|^2] \,,
    \]
for all $f:\mR^d\to \mR$ such that $f$ is continuously differentiable and $\nabla f$ is square integrable with respect to $\pi$, . For any $i \in \{1,2,\ldots,d\}$ and $x = (x_1,x_2,\ldots,x_d)$, let $f(x) = x_i$. Then $\| \nabla f(x) \| = 1$, for all $x\in \mR^d$. Thus, by Poincar\'e inequality, $\Var_\pi(f(X)) = \Var(X_i) \leq C_{\mathsf{PI}}$. Now, $\Var(\pi) = \sum_{i=1}^d \Var(X_i) \leq dC_{\mathsf{PI}}$. Finally, since $\pi$ satisfying the LSI implies that it also satisfies the Poincar\'e inequality with the same constant $C_{\mathsf{LSI}}$, we conclude that
    $\Var(\pi) \leq dC_{\mathsf{LSI}}$.
\end{proof}

\begin{lemma}\label{lem:lsi_lb}
    Consider a probability measure $\pi$ over $\R^d$ with density $\pi(x) \propto \exp(-F(x))$ satisfies the Logarithmic Sobolev Inequality with constant $C_{\mathsf{LSI}}$. Assume that $F$ is $L$-smooth (i.e, $\nabla F$ is L-Lipschitz) and $\nabla F$ is square integrable with respect to $\pi$. Then, $$C_{\mathsf{LSI}} \geq \frac{1}{L}$$
\end{lemma}
\begin{proof}
    Let $x_{\pi} \in \R^d$ be the mean of $\pi$ (a fixed vector). Without loss of generality, we assume $\pi(x) = e^{-F(x)}$. This requires adding a constant to $F(x)$ in the original definition, which does not change $\nabla F$. Using integration by parts, we have:
    \begin{align}
        d &= \int \langle x-x_{\pi},\nabla F(x)\rangle e^{-F(x)} dx \nonumber \\
        &\leq \sqrt{\Var(\pi) }\sqrt{\E_{\pi}\|\nabla F\|^2} \tag{Cauchy-Schwarz Inequality} \nonumber\\
        &\leq \sqrt{d C_{\mathsf{LSI}}}\sqrt{dL} \tag{Lemma~\ref{lem:stationarity_bounds} and Lemma~\ref{lem:energy_gradient}} 
    \end{align}
The claim follows from the above equation.
\end{proof}
\section{Proof of Technical Lemmas}

\subsection{Proof of Lemma~\ref{lem:errorbounds}}
\label{sec:errorbounds_proof}
\begin{proof}
    \begin{enumerate}
        \item First, we bound $\E_{Y,\xi} [\de_1(t)]$. Moving the expectation out of the inner product, applying Cauchy-Schwarz inequality, using the assumption that the function $(x,y) \to u(x,y,\xi)$ is $L_u$-Lipschitz, and the fact that $X_t = X_0 + tu(X_0,Y,\xi) + \sigma B_t$, we get:
\begin{align}
    &\E_{Y,\xi} [\de_1(t)]  = \E \langle \nabla_\W \bar{\sE}(X_t,\rho_t(\cdot |Y,\xi)),u(X_0,Y,\xi) - u(X_t,Y,\xi) \rangle \nonumber \\
    &\leq \E\left[\| \nabla_\W \bar{\sE}(X_t,\rho_t(\cdot |Y,\xi)) \| \cdot \|u(X_0,Y,\xi) - u(X_t,Y,\xi) \|\right] \nonumber \\
    &\leq L_u\E \left[\| \nabla_\W \bar{\sE}(X_t,\rho_t(\cdot |Y,\xi)) \| \cdot \|X_t - X_0 \| \right] \nonumber \\
    &= L_u\E \left[\| \nabla_\W \bar{\sE}(X_t,\rho_t(\cdot |Y,\xi)) \| \cdot \|tu(X_0,Y,\xi) + \sigma B_t \| \right] \nonumber \\
    &\leq L_u \sqrt{\E \| \nabla_\W \bar{\sE}(X_t,\rho_t(\cdot |Y,\xi)) \|^2} \sqrt{\E \|tu(X_0,Y,\xi) + \sigma B_t \|^2} \nonumber \\
    &= L_u \sqrt{\E \| \nabla_\W \bar{\sE}(X_t,\rho_t(\cdot |Y,\xi)) \|^2} \sqrt{t^2\E\|u(X_0,Y,\xi)\|^2 + \sigma^2 td} \label{Eq:GeneralBounsDE1Prelim} \,.
\end{align}

Next, we bound $\E\|u(X_0,Y,\xi)\|^2$. Let $(X^\ast,Y^\ast)\sim \pi\times \pi$ be optimally coupled to $(X_0,Y)\sim \rho_0 \times \rho_0$ in the 2-Wasserstein distance and independent of $\xi$. Thus, by the triangle inequality, the inequality $(a + b)^2 \leq 2a^2 + 2b^2$, by $L_u$-Lipschitz continuity of $(x,y) \to u(x,y,\xi)$, and the fact that $\W_2^2(\mu\times \mu, \nu\times \nu) \leq 2\W_2^2(\mu, \nu)$, for any probability measures $\mu, \nu$, we have:
\begin{align}
    \E \|u(X_0,Y,\xi)\|^2 &\leq 2\E \|u(X_0,Y,\xi) - u(X^\ast,Y^\ast,\xi) \|^2 + 2\E \|u(X^\ast,Y^\ast,\xi)\|^2 \nonumber \\
    &\leq 2L_u^2\E [\|X_0 - X^\ast\|^2 + \|Y - Y^\ast\|^2] + 2\E \|u(X^\ast,Y^\ast,\xi)\|^2 \nonumber \\
    &= 2L_u^2 \W_2^2(\rho_0 \times \rho_0 , \pi\times \pi) + 2\E \|u(X^\ast,Y^\ast,\xi)\|^2 \nonumber \\
    &\leq 4L_u^2 \W_2^2(\rho_0, \pi) + 2\E\|u(X^\ast,Y^\ast,\xi)\|^2 \label{Eq:GeneralBoundDE1Prelim1} \,.
\end{align}
Next, using $\E[u(x,Y^{*},\xi)] = -\nabla_\W \sF(x,\pi)$, for every $x$, we obtain: 
\begin{align}
    \E \|u(X^*,Y^*,\xi)\|^2 &= \E \| u(X^*, Y^*, \xi) + \nabla_\W \sF(X^*,\pi) \|^2 + \E \| \nabla_\W \sF(X^*,\pi) \|^2 \nonumber \\
    &= (\sigma^*)^2 + (G_{\pi})^2 \label{Eq;GeneralBoundDE1Prelim2} \,.
\end{align}
Now by using the bounds \eqref{Eq:GeneralBoundDE1Prelim1} and \eqref{Eq;GeneralBoundDE1Prelim2} in \eqref{Eq:GeneralBounsDE1Prelim} proves \eqref{Eq:Lemma2BoundDE_1}. 
 \item Next, we bound $\E_{Y,\xi}[\de_2(t)]$. Using the assumption $\|\nabla_\W \sF(x,\mu)-\nabla_\W\sF(x,\nu)\| \leq L_\sF \W_2(\mu,\nu)$, applying the Cauchy-Schwarz inequality and Jensen's inequality, we get:
\begin{align}
    \E_{Y,\xi}[\de_2(t)] &\leq L_\sF \E\left[\W_2(\rho_t(\cdot | Y,\xi) , \rho_0) \| \nabla_\W \bar{\sE}(X_t,\rho_t(\cdot |Y,\xi)) \|\right] \nonumber \\
    &\leq L_\sF\sqrt{\E[\W_2^2(\rho_t(\cdot | Y,\xi) , \rho_0)]} \cdot \sqrt{\E \| \nabla_\W \bar{\sE}(X_t,\rho_t(\cdot |Y,\xi)) \|^2}  \label{Eq:GeneralBoundDE2Prelim1} \,.
\end{align}
Next, we bound $\E[\W_2^2(\rho_t(\cdot | Y,\xi) , \rho_0)]$. Since $X_t|Y,\xi \sim \rho_t(\cdot| Y,\xi)$, $X_0 \sim \rho_0$ and $X_t =  X_0 + t u(X_0,Y,\xi) + \sigma B_t$. Therefore, by the definition of the Wasserstein distance, we have:
\begin{align}
    \E [\W_2^2(\rho_t(\cdot | Y,\xi) , \rho_0)] &\leq \E[\E[\|  X_t - X_0 \|^2|Y,\xi]] \nonumber \\
    &= \E \| t u(X_0, Y, \xi) + \sigma B_t \|^2 \nonumber \\
    &= t^2 \E_{(X,Y,\xi) \sim \rho_{0}\times \rho_0\times \nu} \|u(X,Y,\xi)\|^2 + \sigma^2 td \nonumber \\
    &\leq 4t^2L_u^2 \W_2^2(\rho_0, \pi) + 2t^2( (\sigma^*)^2 + (G_{\pi})^2) + \sigma^2 td \label{Eq:GeneralBoundDE2Prelim2}\,,
\end{align}
where the last inequality follows from \eqref{Eq:GeneralBoundDE1Prelim1} and \eqref{Eq;GeneralBoundDE1Prelim2}. By plugging the bound in \eqref{Eq:GeneralBoundDE2Prelim2} into \eqref{Eq:GeneralBoundDE2Prelim1} proves \eqref{Eq:Lemma2BoundDE_2}.

\item Next, we bound $\E_{Y,\xi}[\se(t)]$. Define $\Theta(x,y,\xi) :=  u(x,y,\xi) + \nabla_\W\sF(x,\rho_0)$. Note that since $\E[u(x,Y,\xi)] = -\nabla_\W \sF(x,\pi)$, for every $x$, we have:
\begin{align}
    \E[\nabla_x \cdot \Theta(x,Y,\xi)]  &= \E [\nabla_x \cdot u(x,Y,\xi) +  \nabla_x \cdot \nabla_\W \sF(x,\rho_0)] \nonumber \\
    &= \nabla_x \cdot \E [u(x,Y,\xi) +  \nabla_\W \sF(x,\rho_0)] = 0 \label{Eq:GBSEPrelim1} \,.
\end{align}
Here, we have exchanged the integral and derivative using the dominated convergence theorem along with Lipchitz continuity of $u,\nabla_\W\sF$. 
Using $\nabla_\W \bar{\sE}(x,\rho_t(\cdot |y,\xi)) = \nabla_\W \bar{\sF}(x,\rho_t(\cdot |y,\xi)) +\frac{\sigma^2}{2}\nabla \log \rho_t(x|y,\xi)$ and integrating by parts, we get:
\begin{align}
    \E_{Y,\xi} [\se(t)] &= \E \innerrho{\rho_t(\cdot |Y,\xi)}{\nabla_\W \bar{\sF}(\cdot,\rho_t(\cdot |Y,\xi))}{ \Theta(\cdot,Y,\xi)} \nonumber \\
    &+ \frac{\sigma^2}{2} \E\int \rho_t(x | Y,\xi) \langle \nabla_x \log \rho_t(x | Y,\xi) \,,\, \Theta(x,Y,\xi) \rangle dx \nonumber \\
    &= \E \innerrho{\rho_t(\cdot |Y,\xi)}{\nabla_\W \bar{\sF}(\cdot,\rho_t(\cdot |Y,\xi))}{ \Theta(\cdot,Y,\xi)}
    \nonumber \\
    &+ \frac{\sigma^2}{2} \E\int \langle \nabla_x \rho_t(x | y,\xi) \,,\, \Theta(x,y,\xi) \rangle dx \text{, \quad \quad (since $\nabla \log p_t = \tfrac{\nabla p_t}{p_t}$)} \nonumber \\
    &= \E \innerrho{\rho_t(\cdot |Y,\xi)}{\nabla_\W \bar{\sF}(\cdot,\rho_t(\cdot |Y,\xi))}{ \Theta(\cdot,Y,\xi)}
    \nonumber \\
    &- \frac{\sigma^2}{2} \E\int (\nabla_x \cdot \Theta(x,Y,\xi)) \rho_t(x | Y,\xi) dx \text{, \quad \quad \quad (integration by parts)}\label{Eq:GBSEPrelim2} \,.
\end{align}
Now we bound the term $\E\int (\nabla_x \cdot \Theta(x,Y,\xi)) \rho_t(x | Y,\xi) dx$. From \eqref{Eq:GBSEPrelim1} it follows that
\begin{align*}
    \E\int (\nabla_x \cdot \Theta(x,Y,\xi)) \rho_t(x | y,\xi) dx &= \E\int (\nabla_x \cdot \Theta(x,Y,\xi)) (\rho_t(x | Y,\xi) - \rho_t(x)) dx \,.
\end{align*}
Since the functions $x\to u(x,y,\xi)$ and $x\to \nabla_{\W}\sF(x,\mu)$ are continuously differentiable and $L_u$-Lipschitz continuous, $\Theta(x,y,\xi)$ is 2$L_u$-Lipschitz continuous and continuously differentiable. Thus $|\nabla_x \cdot \Theta(x,y,\xi)| \leq 2dL_u$ uniformly. By noting $\frac{1}{2}\int |\rho_t(x | y,\xi) - \rho_t(x)| dx$ is the total-variation distance between $\rho_t(\cdot| y,\xi)$ and $\rho_t$, and  applying Pinsker's inequality, we get:
\begin{align}
    \bigr|\E\int (\nabla \cdot \Theta(x,Y,\xi)) (\rho_t(x | Y,\xi) - \rho_t(x)) dx\bigr| &\leq 
    \E\int |\nabla \cdot \Theta(x,Y,\xi)| \cdot |\rho_t(x | Y,\xi) - \rho_t(x)| dx \nonumber \\
    &\leq 4dL_u \E[\mathsf{TV}(\rho_t,\rho_t(\cdot |Y,\xi))]\nonumber \\
    &\leq 4dL_u \E\sqrt{\frac{1}{2}\KL(\rho_t(\cdot \mid Y,\xi)  \,||\, \rho_t)} \nonumber \\
    &= \sqrt{8}dL_u \E\sqrt{\KL(\rho_t(\cdot \mid Y,\xi)  \,||\, \rho_t)} \label{Eq:GBSEPrelim3}\,.
\end{align}
Since $\rho_t = \E [\rho_t(\cdot \mid Y,\xi)]$, we consider $Y',\xi'$ to be an i.i.d. copy of $(Y,\xi)$ respectively and to be independent of $X_0$. Since the KL divergence functional is convex jointly in its arguments, we have:
\begin{align}
\KL(\rho_t(\cdot \mid Y,\xi) \,||\, \rho_t) &= \KL(\rho_t(\cdot \mid Y,\xi) \,||\, \E[\rho_t(\cdot \mid Y',\xi')]) \nonumber \\
&\leq \E\left[\KL(\rho_t(\cdot \mid Y,\xi) \,||\, \rho_t(\cdot \mid Y',\xi')) |Y,\xi\right] \label{Eq:GBSEPrelim4}\,.
\end{align}
Further conditioning on $X_0$ and taking an expectation yields $\rho_t(\cdot|Y,\xi) = \E[\rho_t(\cdot|X_0,Y,\xi)|Y,\xi]$. Another application of the joint convexity of the KL divergence functional in the above inequality gives
\begin{equation}\label{Eq:GBSEPrelim5}
    \E \left[\KL(\rho_t(\cdot|Y,\xi) \,||\,\rho_t(\cdot|Y',\xi'))|Y,\xi \right]\leq \E \left[\KL(\rho_t(\cdot\mid X_0,Y,\xi) \,||\, \rho_t(\cdot \mid X_0, Y',\xi'))|Y,\xi \right] \,.
\end{equation}
Since $X_t =  X_0 + t u(X_0,Y,\xi) + \sigma B_t$,
we have $\rho_t(\cdot|X_0,Y,\xi) = \sN(\mu_1, \sigma^2 tI)$, $\rho_t(\cdot| X_0,Y',\xi') = \sN(\mu_2, \sigma^2 tI)$, where $\mu_1 := X_0 + t u(X_0, Y, \xi), \mu_2 := X_0 + t u(X_0, Y', \xi')$. Using the formula for KL-divergence between multivariate normal distributions in \cite{duchi2007derivations} yields:
\begin{align}
    \KL(\rho_t(\cdot\mid X_0,Y,\xi) \,||\, \rho_t(\cdot \mid X_0, Y',\xi')) &= \frac{\| \mu_1 - \mu_2 \|^2}{2\sigma^2t} \nonumber \\
    &= \frac{t}{2\sigma^2}\| u(X_0,Y,\xi) - u(X_0,Y',\xi') \|^2 \label{Eq:GBSEPrelim6}\,.
\end{align}
From the bounds in \eqref{Eq:GBSEPrelim4}, \eqref{Eq:GBSEPrelim5}, \eqref{Eq:GBSEPrelim6}, and by applying Jensen's inequality to the outer expectation, we obtain
\begin{align}
    \E \sqrt{\KL(\rho_t(\cdot \mid Y,\xi)  \,||\, \rho_t)} &\leq \E \sqrt{ \frac{t}{2\sigma^2} \E\bigr[\| u(X_0,Y,\xi) - u(X_0,Y',\xi') \|^2\bigr|Y,\xi\bigr]} \nonumber \\
    &\leq \frac{\sqrt{t}}{\sigma} \sqrt{\frac{1}{2}\E \| u(X_0,Y,\xi) - u(X_0,Y',\xi') \|^2} \nonumber \\
    &\leq \frac{\sqrt{t}}{\sigma} \sqrt{\E_{x\sim \rho_0}[\Var(u(x,Y,\xi))]}\label{Eq:GBSEPrelim7}\,.
\end{align}
Then by plugging the bound in \eqref{Eq:GBSEPrelim7} into \eqref{Eq:GBSEPrelim3}, we get:
\begin{equation} \label{Eq:GBSEPrelim8}
    \biggr|\E\int (\nabla_x \cdot \Theta(x,Y,\xi)) (\rho_t(x | Y,\xi) - \rho_t(x)) dx\biggr| \leq \frac{\sqrt{8}\sqrt{t} L_u d}{\sigma} \sqrt{\E_{x\sim \rho_0} [\Var(u(x,Y,\xi))]} \,.
\end{equation}
Next, we bound the first term in \eqref{Eq:GBSEPrelim2}. First, note that from \eqref{Eq:GBSEPrelim1}, we have:
\[
\E\int \langle \nabla_\W \bar{\sF}(x,\rho_t) \,,\, \Theta(x,Y,\xi) \rangle \rho_t(x) dx = 0 \,.
\]
Using this fact, we obtain:
\begin{align*}
    &\E \langle \nabla_\W \bar{\sF}(X_t,\rho_t(\cdot |Y,\xi)) \,,\, \Theta(X_t,Y,\xi) \rangle_{L_2(\rho_t(\cdot | Y,\xi))} \\
    &= \E \langle \nabla_\W \bar{\sF}(X_t,\rho_t(\cdot |Y,\xi)) \,,\, \Theta(X_t,Y,\xi) \rangle_{L_2(\rho_t(\cdot | Y,\xi))} - \E\int \langle \nabla_\W \bar{\sF}(x,\rho_t) \,,\, \Theta(x,Y,\xi) \rangle \rho_t(x) dx  \\
    &= \E \int \langle \nabla_\W \bar{\sF}(x,\rho_t(\cdot |Y,\xi)) \,,\, \Theta(x,Y,\xi) \rangle \rho_t(x|Y,\xi) dx\nonumber \\
    &\quad - \E\int \langle \nabla_\W \bar{\sF}(x,\rho_t) \,,\, \Theta(x,Y,\xi)\rangle \rho_t(x) dx\,.
\end{align*}
For a given $Y,\xi$, let $Z_1 \sim \rho_t(\cdot | Y,\xi)$ and $Z_2\sim \rho_t$ be optimally coupled in the 2-Wasserstein distance. By the Cauchy-Schwarz inequality we conclude that almost surely $Y,\xi$:
\begin{align}
    & \E_{x\sim \rho_t(\cdot | Y,\xi)} \langle \nabla_\W \bar{\sF}(x,\rho_t(\cdot |y,\xi)) \,,\, \Theta(x,Y,\xi) \rangle 
    - \E_{x\sim \rho_t} \langle \nabla_\W \bar{\sF}(x,\rho_t) \,,\, \Theta(x,Y,\xi) \rangle \nonumber \\
    &\quad = \E\langle \nabla_\W \bar{\sF}(Z_1, \rho_t(\cdot | y, \xi)) \,,\, \Theta(Z_1,Y,\xi) \rangle - \E\langle \nabla_\W \bar{\sF}(Z_2,\rho_t) \,,\, \Theta(Z_2,Y,\xi) \rangle \nonumber \\
    &\quad = \E\langle \nabla_\W \bar{\sF}(Z_1, \rho_t(\cdot | Y, \xi)) - \nabla_\W \bar{\sF}(Z_2, \rho_t) \,,\, \Theta(Z_1,Y,\xi) \rangle \nonumber \\
    &\quad + \E\langle \nabla_\W \bar{\sF}(Z_2, \rho_t) \,,\, \Theta(Z_1,y,\xi) - \Theta(Z_2,y,\xi)\rangle \nonumber \\
    &\quad \leq \E\big\| \nabla_\W \bar{\sF}(Z_1, \rho_t(\cdot | Y, \xi)) - \nabla_\W \bar{\sF}(Z_2, \rho_t) \big\| \cdot \big\|\Theta(Z_1,Y,\xi) \big\| \nonumber \\
    &\quad + \E \big\| \nabla_\W \bar{\sF}(Z_2, \rho_t) \big\| \cdot \big\|\Theta(Z_1,Y,\xi) - \Theta(Z_2,Y,\xi)\big\| \label{Eq:GBSEPrelim9} \,.
\end{align}
Next, note that by the definition of $\Theta(x, y, \xi)$, the assumption that $(x,y) \to u(x,y,\xi)$ and $x\to \nabla_{\W}\sF(x,\mu)$ are $L_u$-Lipschitz, and the assumption $\|\nabla_\W \bar{\sF}(x,\mu) - \nabla_\W \bar{\sF}(y,\nu)\| \leq L_{\bar{u}} \|x - y\| + L_{\bar{\sF}} \W_2(\mu,\nu)$, we obtain:
\[
    \big\|\Theta(X_1,Y,\xi) - \Theta(X_2,Y,\xi)\big\| \leq 2L_u \|X_1 - X_2\| \,,
\]
\[
 \big\| \nabla_\W \bar{\sF}(X_1, \rho_t(\cdot | Y, \xi)) - \nabla_\W \bar{\sF}(X_2, \rho_t) \big\| \leq L_{\bar{u}}\|X_1-X_2\| + L_{\bar{\sF}}\W_2(\rho_t, \rho_t(\cdot | Y,\xi)) \,.
\]
Using the above bounds in \eqref{Eq:GBSEPrelim9}, applying the Cauchy-Schwarz inequality, and the Jensen's inequality yields the following almost surely $Y,\xi$:
\begin{align*}
    & \E_{x\sim \rho_t(\cdot | Y,\xi)} \langle \nabla_\W \bar{\sF}(x,\rho_t(\cdot |Y,\xi)) \,,\, \Theta(x,Y,\xi) \rangle 
    - \E_{x\sim \rho_t} \langle \nabla_\W \bar{\sF}(x,\rho_t) \,,\, \Theta(x,Y,\xi) \rangle \\
    &\quad \leq \sqrt{2}(L_{\bar{u}}+L_{\bar{\sF}}) \W_2(\rho_t, \rho_t(\cdot | Y,\xi)) \sqrt{\E_{x\sim \rho_t(\cdot |Y,\xi)}\big\|\Theta(x,Y,\xi) \big\|^2 } \\&\quad 
    + 2L_u \W_2(\rho_t, \rho_t(\cdot | Y,\xi)) \sqrt{\E_{x\sim \rho_t}\big\| \nabla_\W \bar{\sF}(x, \rho_t) \big\|^2 } \,.
\end{align*}
Hence, by taking expectation with respect to $Y, \xi$ and another application of the Cauchy-Schwarz inequality, we obtain:
\begin{align}
    &\E\langle \nabla_\W \bar{\sF}(X_t,\rho_t(\cdot |Y,\xi)) \,,\, \Theta(X_t,Y,\xi) \rangle_{L_2(\rho_t(\cdot | Y,\xi))} \nonumber \\
    &\quad \leq 2\sqrt{\E [\W_2^2(\rho_t, \rho_t(\cdot | Y,\xi))]} \left[(L_{\bar{u}}+L_{\bar{\sF}}) \sqrt{\E\big\|\Theta(X_t,Y,\xi) \big\|^2} + L_u \sqrt{\E \big\| \nabla_\W \bar{\sF}(X_t, \rho_t) \big\|^2}\right] \,.\label{Eq:GBSEPrelim10}
\end{align}
Finally, by multiplying the bound in \eqref{Eq:GBSEPrelim8} by $\sigma^2/2$, adding the resultant to \eqref{Eq:GBSEPrelim10}, and using \eqref{Eq:GBSEPrelim2}
we prove \eqref{Eq:Lemma2BoundSE}.
\item Finally, we bound $\E_{y,\xi}[\le(t)]$. By applying the Cauchy-Schwartz inequality twice and the assumption $\|\nabla_\W \bar{\sF}(x,\mu)-\nabla_\W \sF(x,\mu)\| \leq L_l\W_2(\pi,\mu)$, we have almost surely $Y,\xi$:
\begin{align}
    \le(t) &= \E_{x\sim\rho_t(\cdot | Y,\xi)} \langle \nabla_\W \bar{\sE}(x,\rho_t(\cdot |Y,\xi)), \nabla_\W \bar{\sF}(x,\rho_t(\cdot |Y,\xi)) - \nabla_\W \sF(x,\rho_t(\cdot |Y,\xi) \rangle \nonumber \\
    &\leq \E_{x\sim \rho_t(\cdot | Y,\xi)} \left[ \| \nabla_\W \bar{\sE}(x,\rho_t(\cdot |Y,\xi)) \| \cdot \|\nabla_\W \bar{\sF}(x,\rho_t(\cdot |Y,\xi)) - \nabla_\W \sF(x,\rho_t(\cdot |Y,\xi) \| \right] \nonumber \\
    &\leq  L_l \E_{x\sim\rho_t(\cdot | Y,\xi)} \left[\W_2(\pi, \rho_t(\cdot | Y,\xi)) \| \nabla_\W \bar{\sE}(x,\rho_t(\cdot |Y,\xi)) \|\right] \nonumber \\
    &\leq L_l \sqrt{ \W_2^2(\pi, \rho_t(\cdot | Y,\xi))} \sqrt{\E_{x\sim\rho_t(\cdot | Y,\xi)} \| \nabla_\W \bar{\sE}(x,\rho_t(\cdot |Y,\xi)) \|^2} \label{Eq:GBLE} \,.
\end{align}
Next, by taking an expectation with respect to $(Y,\xi)\sim \rho_0\times \nu$ and applying the Jensen's inequality, we obtain \eqref{Eq:Lemma2BoundLE}.
    \end{enumerate}
\end{proof}

\subsection{Proof of Lemma~\ref{lem:descent_lemma}}
\label{subsec:descent_lemma_proof}
Before delving into the proof of Lemma~\ref{lem:descent_lemma}, we note the following bound on the stochastic error.

\begin{lemma}[Stochastic error bound]
\label{lem:Stochasticerrorbound}
Define
$V(u,\rho_0) := \E_{x\sim \rho_0}[\Var(u(x,Y,\xi))]$ and $G_{\mathsf{mod}}^2 := \E_{X^{*}\sim \pi}\|\nabla_{\W}\bar{\sF}(X^*,\pi)\|^2$. 
Let Assumptions~\ref{Assumption:Lipschitz},~\ref{Assumption:functional} and~\ref{Assumption:BoundVariance} hold. Assume $L_u t \leq 1$. Then for arbitrary $\beta > 0$, we have:

\begin{align*}
\E[\se(t)] &\lesssim (\sigma L_u d \sqrt{t}+t L_u G_{\mathsf{mod}})\sqrt{(C^{\Var}\Var({\pi}) + C_{\nu}^{\Var})}  + t(L_{\bar{u}} + L_{\bar{\sF}})(C^{\Var}\Var(\pi) + C_{\nu}^{\Var})\nonumber \\
&\quad + \frac{(\sigma L_u d \sqrt{t}+t L_u G_{\mathsf{mod}})^2C^{\Var}C_{\mathsf{LSI}}C_{\KL}}{\beta}  + (t C^{\Var} (L_{\bar{u}}+ L_{\bar{\sF}}) C_{\mathsf{LSI}}C_{\KL}+\beta)\bar{\sE}(\rho_0)
 \nonumber \\ &\quad +  t L^2_u (L_{\bar{u}}+ L_{\bar{\sF}}) C_{\mathsf{LSI}}C_{\KL}\E\bar{\sE}(\rho_t(\cdot |Y,\xi)) \,.
\end{align*}
\end{lemma}

\begin{proof}
 We define the following to reflect the result of Lemma~\ref{lem:errorbounds}:

\begin{align*}
    &\Theta(x,y,\xi) := u(x, y, \xi) + \nabla_\W \sF(x,\rho_0) \\
    &T_4 := 2\sqrt{\E [\W_2^2(\rho_t, \rho_t(\cdot | Y,\xi))]} \left[(L_{\bar{u}} + L_{\bar{\sF}}) \sqrt{\E \big\| \Theta(X_t,Y,\xi) \big\|^2} + L_u \sqrt{\E\big\| \nabla_\W \bar{\sF}(X_t, \rho_t) \big\|^2}\right] \\
    &T_5 := \sigma\sqrt{2t} L_u d \sqrt{\E_{x\sim \rho_0} [\Var(u(x,Y,\xi)]}  \,.
\end{align*}
From Lemma \ref{lem:errorbounds}, we have $\E [\se(t)] \leq T_4 + T_5$. We now bound each term of $T_4(t)$. By Lemma~\ref{lem:cond_wass_bound}, we have:

\begin{equation}
\label{Eq:T4Term1Bound}
    \sqrt{\E \W_2^2(\rho_t, \rho_t(\cdot | Y,\xi))} \leq \sqrt{2t^2 V(u,\rho_0)} \,.
\end{equation}

For the second term, given $Y,\xi$, let $Z_1 \sim \rho_t(\cdot | Y,\xi)$ and $Z_2\sim \rho_t$ be optimally coupled in the 2-Wasserstein distance. By Assumption \ref{Assumption:Lipschitz}, we obtain:
\begin{align*}
    \big\| \Theta(Z_1,Y,\xi) \big\| &\leq \big\| \Theta(Z_1,Y,\xi) - \Theta(Z_2,y,\xi) \big\| + \big\| \Theta(Z_2,Y,\xi) \big\| \\
    &\leq 2L_u\| Z_1 - Z_2 \| + \big\| \Theta(Z_2,Y,\xi) \big\| \,.
\end{align*}
Next, by squaring, applying the inequality $(a + b)^2 \leq 2a^2 + 2b^2$, taking expectation, applying Lemma~\ref{lem:cond_wass_bound}, we get almost surely $Y,\xi$:
\begin{align*}
    \E[\big\| \Theta(Z_1,Y,\xi) \big\|^2|Y,\xi] &\leq 8L_u^2 \W_2^2(\rho_0, \rho_t(\cdot | Y,\xi)) + 2\E[\big\|\Theta(Z_2,Y,\xi) \big\|^2|Y,\xi] \,.
\end{align*}
Applying Lemma~\ref{lem:cond_wass_bound}, Lemma~\ref{lem:var_control} and Assumption \ref{Assumption:BoundVariance}, after noting that $X_t|Y,\xi \stackrel{d}{=}Z_1$ we obtain:
\begin{align}
     \E\big\| \Theta(X_t,Y,\xi) \big\|^2 &\leq 16L_u^2 t^2 V(u,\rho_0) + 2\E[\E[\big\|\Theta(Z_2,Y,\xi) \big\|^2|Y,\xi]] \nonumber\\
     &\leq  16L_u^2 t^2 V(u,\rho_0) + 2C^{\Var}\Var(\rho_0) + 2C_{\nu}^{\Var} \tag{ By Assumption~\ref{Assumption:BoundVariance}} \nonumber\\
     &\lesssim C^{\Var} C_{\mathsf{LSI}} C_{\KL} \bar{\sE}(\rho_0) + C^{\Var}\Var(\pi) + C^{\Var}_{\nu} \,. \label{Eq:T4Term2Bound}
     \end{align}
Here we have used the assumption that $L_ut \leq 1$. The last step above follows by an application of Lemma~\ref{lem:var_control}. Now, additionally consider Assumptions~\ref{Assumption:functional} and~\ref{Assumption:BoundVariance}. We apply Lemma~\ref{lem:control_from_pi}, to conclude:
\begin{align*}
    \E_{x\sim \rho_t}\big\| \nabla_\W \bar{\sF}(x, \rho_t) \big\|^2 &\leq 3(L_{\bar{\sF}}^2+L_{\bar{u}}^2)\W_2^2(\rho_t, \pi) + 3 \E_{x\sim \pi}\big\| \nabla_\W \bar{\sF}(x, \pi) \big\|^2 \,. \nonumber\\
    &\leq 3(L_{\bar{\sF}}^2+L_{\bar{u}}^2)\E\W_2^2(\rho_t(\cdot |Y,\xi), \pi) + 3 \E_{x\sim \pi}\big\| \nabla_\W \bar{\sF}(x, \pi) \big\|^2 \,. \
\end{align*}
The last step follows from the usual convexity of $\W_2^2$. By applying the Talagrand's $T_2$-inequality implied by Lemma~\ref{lem:OttoVillaniTheorem} and Assumption \ref{Assumption:functional}-\eqref{eq:LSI} along with Assumption~\ref{Assumption:functional}-\eqref{eq:kl_growth}, we get:
\begin{align}\label{Eq:T4Term3Bound}
    \E_{x\sim \rho_t}\big\| \nabla_\W \bar{\sF}(x, \rho_t) \big\|^2 &\lesssim (L_{\bar{\sF}}^2 +L_{\bar{u}}^2)C_{\mathsf{LSI}} C_{\KL} \E\bar{\sE}(\rho_t(\cdot |Y,\xi)) +  \E_{x\sim \pi}\big\| \nabla_\W \bar{\sF}(x, \pi) \big\|^2 \,.
\end{align}
Let $\zeta(\rho_0):=\sqrt{C^{\Var}C_{\mathsf{LSI}} C_{\KL} \bar{\sE}(\rho_0) + C^{\Var}\Var(\pi) + C^{\Var}_{\nu}}$. We now use equations~\eqref{Eq:T4Term3Bound}, 
\eqref{Eq:T4Term2Bound} and ~\eqref{Eq:T4Term1Bound}, along with Lemma~\ref{lem:var_control} to bound $V(u,\rho_0)$ and hence $\E [\se(t)]$ as:
\begin{align*}
\E[\se(t)] &\lesssim  (\sigma L_u d \sqrt{t}+tL_u G_{\mathsf{mod}})\zeta(\rho_0)   + t(L_{\bar{u}} + L_{\bar{F}})\zeta^2(\rho_0) \nonumber \\ &\quad +  t L_u (L_{\bar{u}}+ L_{\bar{\sF}})\zeta(\rho_0)\sqrt{ C_{\mathsf{LSI}}C_{\KL}\E\bar{\sE}(\rho_t(\cdot |Y,\xi))} \,.
\end{align*}
Applying AM-GM inequality on $t L_u (L_{\bar{u}}+ L_{\bar{\sF}})\zeta(\rho_0)\sqrt{ C_{\mathsf{LSI}}C_{\KL}\E\bar{\sE}(\rho_t(\cdot |Y,\xi))}$, we conclude:
\begin{align*}
\E[\se(t)] &\lesssim  (\sigma L_u d \sqrt{t}+tL_u G_{\mathsf{mod}})\zeta(\rho_0)   + t(L_{\bar{u}} + L_{\bar{\sF}})(C^{\Var}\Var(\pi) + C_{\nu}^{\Var}) \nonumber \\ &\quad +  t L^2_u (L_{\bar{u}}+ L_{\bar{\sF}}) C_{\mathsf{LSI}}C_{\KL}\E\bar{\sE}(\rho_t(\cdot |Y,\xi)) + tC^{\Var} (L_{\bar{u}}+ L_{\bar{\sF}}) C_{\mathsf{LSI}}C_{\KL}\bar{\sE}(\rho_0) \,.
\end{align*}
Now, using the fact that $\sqrt{x+y} \leq \sqrt{x} + \sqrt{y}$ for every $x,y \geq 0$ on $\zeta(\rho_0)$, we conclude: 
\begin{align*}
(\sigma L_u d \sqrt{t}+tL_u G_{\mathsf{mod}})\zeta(\rho_0)  &\leq (\sigma L_u d \sqrt{t}+tL_u G_{\mathsf{mod}})\sqrt{(C^{\Var}\Var(\pi) + C_{\nu}^{\Var})} \nonumber \\
&\quad + (\sigma L_u d \sqrt{t}+tL_u G_{\mathsf{mod}})\sqrt{C^{\Var}C_{\mathsf{LSI}}C_{\KL}\bar{\sE}(\rho_0)} \,.
\end{align*}

Plugging this into the bound for $\E \se(t)$ above and applying the AM-GM inequality on $(\sigma L_u d \sqrt{t}+tL_u G_{\mathsf{mod}})\sqrt{C^{\Var}C_{\mathsf{LSI}}C_{\KL}\bar{\sE}(\rho_0)}$, we conclude the result.
\end{proof}

\begin{proof}[Proof of Lemma~\ref{lem:descent_lemma}]
Let $t \in [0,\eta]$. First, we consider the error terms $\de_1,\de_2,\se,\le$ in \eqref{eq:time_evolution} and obtain:
\begin{align}
    \frac{d}{dt}\bar{\sE}(\rho_t(\cdot | Y,\xi)) &= - \int dx\rho_t(x | Y,\xi)\|\nabla_{\W}\bar{\sE}(x,\rho_t(\cdot | Y,\xi))\|^2 + \de_1(t) + \de_2(t) + \se(t) + \le(t)  \nonumber \\
    &= - \frac{1}{2}\int dx\rho_t(x | Y,\xi)\|\nabla_{\W}\bar{\sE}(x,\rho_t(\cdot | Y,\xi))\|^2 + \mathsf{res}(t) \label{Eq:DescentLemmaPrelim1}\,,
\end{align}
where $\mathsf{res}(t):= \de_1(t) + \de_2(t) + \se(t) + \le(t) - \frac{1}{2}\int \rho_t(x | Y,\xi)\|\nabla_{\W}\bar{\sE}(x,\rho_t(\cdot | Y,\xi))\|^2 dx$. We begin by bounding the linearization error $\le(t)$ using the assumptions stated in this lemma. From \eqref{Eq:GBLE}, the observation that $\pi$ satisfies Talagrand's $T_2$-inequality by Lemma \ref{lem:OttoVillaniTheorem}, Assumption \ref{Assumption:functional}-\eqref{eq:kl_growth}, and the inequality $ab \leq \frac{a^2}{4} + b^2$ for all $a,b\in \mR$, we get:
\begin{align*}
    \le(t) &\leq L_l\sqrt{2C_{\mathsf{LSI}}} \sqrt{ \KL(\pi, \rho_t(\cdot | Y,\xi))} \sqrt{\E_{\rho_t(\cdot | Y,\xi)} \| \nabla_\W \bar{\sE}(x,\rho_t(\cdot | Y,\xi)) \|^2} \\
    &\leq  L_l\sqrt{2C_{\KL}C_{\mathsf{LSI}}} \sqrt{ \bar{\sE}(\rho_t(\cdot | Y,\xi))} \sqrt{\E_{\rho_t(\cdot | Y,\xi)} \| \nabla_\W \bar{\sE}(x,\rho_t(\cdot | Y,\xi)) \|^2} \\
    &\leq 2L_l^2 C_{\KL}C_{\mathsf{LSI}} \bar{\sE}(\rho_t(\cdot | Y,\xi)) + \frac{1}{4} \int \rho_t(x | Y,\xi)\|\nabla_{\W}\bar{\sE}(x,\rho_t(\cdot | Y,\xi))\|^2 dx \\
    &\leq 2L_l^2 C_{\KL}C_{\mathsf{LSI}} \bar{\sE}(\rho_t(\cdot | Y,\xi)) + \frac{1}{4} \int \rho_t(x | Y,\xi)\|\nabla_{\W}\bar{\sE}(x,\rho_t(\cdot | Y,\xi))\|^2 dx \,.
\end{align*}
By plugging the above inequality into \eqref{Eq:DescentLemmaPrelim1} and applying Assumption \ref{Assumption:functional}-\eqref{eq:PL} with $\mu = \rho_t(\cdot | Y,\xi)$, we get:
\begin{align}
    \frac{d}{dt}\bar{\sE}(\rho_t(\cdot | Y,\xi)) &\leq - \frac{1}{2}\int dx\rho_t(x | Y,\xi)\|\nabla_{\W}\bar{\sE}(x,\rho_t(\cdot | Y,\xi))\|^2 + 2L_l^2 C_{\KL}C_{\mathsf{LSI}} \bar{\sE}(\rho_t(\cdot | Y,\xi)) + \mathsf{res}_1(t) \nonumber \\ 
    &\leq -\frac{C_{\bar{\sE}}}{2}\bar{\sE}(\rho_t(\cdot | Y,\xi)) + 2L_l^2 C_{\KL}C_{\mathsf{LSI}} \bar{\sE}(\rho_t(\cdot | Y,\xi)) + \mathsf{res}_1(t) \nonumber \\
    &\leq -C'\bar{\sE}(\rho_t(\cdot | Y,\xi)) + \mathsf{res}_1(t) 
    \label{Eq:DescentLemmaPrelim2}\,,
\end{align}
where
\small
\[
\mathsf{res}_1(t) := \de_1(t) + \de_2(t) + \se(t) - \frac{1}{4}\int \rho_t(x | Y,\xi)\|\nabla_{\W}\bar{\sE}(x,\rho_t(\cdot | Y,\xi))\|^2 dx \,,\, C' := \frac{C_{\bar{\sE}}}{2} - 2L_l^2 C_{\KL}C_{\mathsf{LSI}} \,.
\]
\normalsize
From Equation~\eqref{Eq:DescentLemmaPrelim2} we get:
\begin{align}
\label{Eq:DescentLemmaPrelim31}
    \frac{d}{dt}\bar{\sE}(\rho_t(\cdot | Y,\xi)) &\leq -C'\bar{\sE}(\rho_t(\cdot | Y,\xi)) + \se(t) + \mathsf{res}_2(t) \,, 
\end{align}
where 
\[
\mathsf{res}_2(t) := \de_1(t) + \de_2(t) - \frac{1}{4}\int \rho_t(x | Y,\xi)\|\nabla_{\W}\bar{\sE}(x,\rho_t(\cdot | Y,\xi))\|^2 dx \,.
\]
The next step is to bound $\E [\se(t)]$. To simplify notation, we define:
\begin{align*}
A_1^{\se}(t,\beta) &:=  (\sigma L_u d \sqrt{t}+t L_u G_{\mathsf{mod}})\sqrt{(C^{\Var}\Var(\pi) + C_{\nu}^{\Var})}  + t(L_{\bar{u}} + L_{\bar{\sF}})(C^{\Var}\Var(\pi) + C_{\nu}^{\Var})\nonumber \\
&\quad + \frac{(\sigma L_u d \sqrt{t}+t L_u G_{\mathsf{mod}})^2C^{\Var}C_{\mathsf{LSI}}C_{\KL}}{\beta}  + (t C^{\Var} (L_{\bar{u}}+ L_{\bar{\sF}}) C_{\mathsf{LSI}}C_{\KL}+\beta)\bar{\sE}(\rho_0) \nonumber\\
A_2^{\se}(t) &:=  t L^2_u (L_{\bar{u}}+ L_{\bar{\sF}}) C_{\mathsf{LSI}}C_{\KL} \,.
\end{align*}
By Lemma~\ref{lem:Stochasticerrorbound}, we conclude that for arbitrary $\beta >0$, we have:
\begin{align}
    \E [\se(t)] &\lesssim A_1^{\se}(t,\beta) + A_2^{\se}(t) \E\bar{\sE}(\rho_t(\cdot |Y,\xi))\label{Eq:DescentLemmaPrelim4} \,,
\end{align}
 Using this notation in Equation~\eqref{Eq:DescentLemmaPrelim31}, we obtain:
\begin{align}
     \frac{d}{dt} \bar{\sE}(\rho_t(\cdot | Y,\xi)) &\leq (-C' + A_2^{\se}(t)) \bar{\sE}(\rho_t(\cdot | Y,\xi)) + \se(t)-A_2^{\se}(t)\bar{\sE}(\rho_t(\cdot | Y,\xi)) + \mathsf{res}_2(t) \nonumber \\
&\leq -C^{''}(\eta) \bar{\sE}(\rho_t(\cdot | Y,\xi)) + \se(t)-A_2^{\se}(t)\bar{\sE}(\rho_t(\cdot | Y,\xi)) + \mathsf{res}_2(t)
\label{Eq:DescentLemmaPrelim5}\,,
\end{align}
where $C^{''}(\eta) := \frac{C_{\bar{\sE}}}{2}-2L_l^2C_{\KL}C_{\mathsf{LSI}} - A_2^{\se}(\eta)$. Note that by our assumptions on $\eta$ and the parameter $L_l$ in the statement of this lemma, for a small enough $c_0$, we must have $C^{''}(\eta) \geq \frac{C_{\bar{\sE}}}{4}$. 

Next, by multiplying both sides of \eqref{Eq:DescentLemmaPrelim5} by $e^{C^{''}(\eta) t}$, we get:
\[
    \frac{d}{dt} \left[e^{C^{''}(\eta)t} \bar{\sE}(\rho_t(\cdot | Y,\xi)) \right] \leq e^{C^{''}(\eta)t}\left[ \se(t)-A_2^{\se}(t)\bar{\sE}(\rho_t(\cdot | Y,\xi))+ \mathsf{res}_2(t) \right] \,.
\]
Thus, by integrating on both sides and taking expectation (along with Fubini's theorem), we obtain:
\begin{align}
\label{Eq:DescentLemmaPrelim6}
    \E\bar{\sE}(\rho_\eta(\cdot | Y,\xi)) &\leq e^{-C^{''}(\eta)\eta}\bar{\sE}(\rho_0) + e^{-C^{''}(\eta)\eta} \int_0^\eta \E\left[\se(t)-A_2^{\se}(t)\bar{\sE}(\rho_t(\cdot | Y,\xi)) + \mathsf{res}_2(t) \right] e^{C^{''}(\eta)t} dt \,.
\end{align}
To simplify notation in the subsequent steps of the proof, we define:
\begin{align*}
    T_1(t) &:= \sqrt{ \E \| \nabla_\W \bar{\sE}(X_t,\rho_t(\cdot | Y,\xi)) \|^2} \\
    A(t) &:= L_u\sqrt{2t^2 \left( 2L_u^2 \W_2^2(\rho_0, \pi) + (\sigma^\ast)^2 + (G_{\pi})^2 \right) + \sigma^2 td} \\
    B(t) &:= L_{\sF}\sqrt{2t^2 \left( 2L_u^2 \W_2^2(\rho_0, \pi) + (\sigma^\ast)^2 + (G_{\pi})^2 \right) + \sigma^2 td} \,.
\end{align*}
Next, by using the inequality $ab \leq \frac{a^2}{8} + 2b^2$ for all $a,b\in \mR$, the bounds for $\de_1(t) + \de_2(t)$ using Lemma \ref{lem:errorbounds}, we get:

\begin{align}
   e^{-C^{''}(\eta)\eta} \int_0^\eta \E\left[ \mathsf{res}_2(t) \right] e^{C^{''}(\eta)t} dt 
    &= e^{-C^{''}(\eta)\eta} \int_0^\eta e^{C^{''}(\eta) t}  \left[ \E [\de_1(t) + \de_2(t)] - \tfrac{1}{4}\E T_1(t)^2 \right] dt \nonumber \\
    &\leq e^{-C^{''}(\eta)\eta} \int_0^\eta e^{C^{''}(\eta) t}  \left[\E T_1(t) (A(t) + B(t)) -\frac{1}{4}\E T_1(t)^2  \right] dt \nonumber \\
    &\leq  2e^{-C^{''}(\eta)\eta} \int_0^\eta e^{C^{''}(\eta) t}\E(A(t)^2 + B(t)^2) dt \nonumber \\
    &\leq 2\int_0^\eta \left[\E A(t)^2 + \E B(t)^2\right] dt \,.
\end{align}
Next, we bound $\int_0^\eta \left[A(t)^2 + B(t)^2 \right] dt$. To further simplify notation, we define:
\[
    T_2 :=  2\left( 2L_u^2 \W_2^2(\rho_0, \pi) + (\sigma^*)^2 + (G_{\pi})^2 \right) \,,\, T_3 := \sigma^2 d \,.
\]
 Therefore, integration yields:
\begin{align}
   e^{-C^{''}(\eta)\eta} \int_0^\eta \E\left[ \mathsf{res}_2(t) \right] e^{C^{''}(\eta)t} dt &\leq 2\int_0^\eta  \left[A(t)^2 + B(t)^2 \right] dt \nonumber \\ &\leq 2(L_u^2 + L_\sF^2) \int_0^\eta (T_2 t^2 + T_3 t)dt \nonumber \\
    &= 2(L_u^2 + L_\sF^2)(\tfrac{T_2\eta^3}{3} + \tfrac{T_3\eta^2}{2}) \nonumber 
    \\
    &\lesssim (L_u^2 + L_{\sF}^2)\left[\eta^3\left[L_u^2\W_2^2(\rho_0,\pi) + (\sigma^\ast)^2 + (G_{\pi})^2\right] + \sigma^2 d\eta^2\right] \nonumber \,.
\end{align}
Now, we apply Assumptions~\ref{Assumption:functional}-\eqref{eq:kl_growth} and Assumption~\ref{Assumption:functional}-\eqref{eq:LSI} along with Lemma~\ref{lem:OttoVillaniTheorem} to upper bound $\W_2^2(\rho_0,\pi)$ in the equation above to conclude:
\begin{align}
   e^{-C^{''}(\eta)\eta} \int_0^\eta \E\left[ \mathsf{res}_2(t) \right] e^{C^{''}(\eta)t} dt 
    &\lesssim (L_u^2 + L_{\sF}^2)\left[\eta^3\left[L_u^2C_{\mathsf{LSI}}C_{\KL}\bar{\sE}(\rho_0) + (\sigma^\ast)^2 + (G_{\pi})^2\right] + \sigma^2 d\eta^2\right]    
    \label{Eq:DescentLemmaPrelim7} \,.
\end{align}
Now we consider:
\begin{align}
    &e^{-C^{''}(\eta)\eta} \int_0^\eta e^{C^{''}(\eta) t} \E[\se(t)-A_2^{\se}(t)\bar{\sE}(\rho_t(\cdot | Y,\xi)) ] dt \nonumber \\
    &\lesssim e^{-C^{''}(\eta)\eta} \int_0^\eta e^{C^{''}(\eta) t} A^{\se}_1(t,\beta)dt\tag{From Equation~\eqref{Eq:DescentLemmaPrelim4}}\nonumber \\
    &\leq \eta A_1^{\se}(\eta,\beta) \nonumber \,,
\label{Eq:DescentLemmaPrelim8}
\end{align}
where we recall 
\begin{align}
    A_1^{\se}(\eta,\beta)   &=  (\sigma L_u d \sqrt{\eta}+\eta L_uG_{\mathsf{mod}})\sqrt{(C^{\Var}\Var{\pi} + C_{\nu}^{\Var})}  + \eta(L_{\bar{u}} + L_{\bar{\sF}})(C^{\Var}\Var{\pi} + C_{\nu}^{\Var})\nonumber \\
    &\quad + \frac{(\sigma L_u d \sqrt{\eta}+\eta L_uG_{\mathsf{mod}})^2C^{\Var}C_{\mathsf{LSI}}C_{\KL}}{\beta}  + (\eta C^{\Var} (L_{\bar{u}}+ L_{\bar{\sF}}) C_{\mathsf{LSI}}C_{\KL}+\beta)\bar{\sE}(\rho_0) \,.
\end{align}
First, note that since for every $x \geq 0$, we have $1-x \leq e^{-x} \leq 1-x+\frac{x^2}{2}$ and as demonstrated above $C^{''}(\eta) \geq \frac{C_{\bar{\sE}}}{4}$. With $\beta = c_0C_{\bar{\sE}}$ for some small enough constant $c_0$ and letting $\eta$ small enough as noted in the statement of this lemma, we conclude:
\[
e^{-C^{''}(\eta)\eta} + (L_u^2 + L_{\sF}^2)\eta^3L_u^2C_{\mathsf{LSI}}C_{\KL} + (\eta^2 C^{\Var} (L_{\bar{u}}+ L_{\bar{\sF}}) C_{\mathsf{LSI}}C_{\KL}+\eta \beta)\leq e^{-\frac{\eta C_{\bar{\sE}}}{8}} \,.
\]
We now plug \eqref{Eq:DescentLemmaPrelim7} and~\eqref{Eq:DescentLemmaPrelim8} along with the above bound into ~\eqref{Eq:DescentLemmaPrelim6} to conclude:
\begin{align*}
    \E\bar{\sE}(\rho_\eta(\cdot | Y,\xi)) &\leq e^{-\tfrac{\eta C_{\bar{\sE}}}{8}}\bar{\sE}(\rho_0) + C\gamma_3 \eta^3 + C\gamma_2\eta^2 + C\gamma_1\eta^{\frac{3}{2}} \,.
\end{align*}

\end{proof}

\section{Proof for Pairwise Interaction Potential}

\subsection{Proof of Lemma~\ref{lem:assumption_pairwise}}
\label{subsec:assumption_pairwise_proof}

\begin{proof}
Note that $\hat{G}$ satisfies the Lipschitz continuity property of Assumption \ref{Assumption:Lipschitz} with $L_u = L_V + L_W$ because for every $x_1, x_2,y_1,y_2 \in \mR^d$:
\[
    \|\hat{G}(x_1,y_1) - \hat{G}(x_2,y_2) \| \leq L_V \| x_1 - x_2 \| + L_W (\| x_1 - x_2 \|+\|y_1-y_2\|) \,.
\]
Next, the functional $\sF$ satisfies the Lipschitz continuity property of Assumption \ref{Assumption:Lipschitz}, with $L_u = L_V+L_W$ and $L_{\sF} = L_W$, because, if $Z_1\sim \mu$ and $Z_2\sim \nu$ are optimally coupled in the 1-Wasserstein distance, then by definition of $\nabla_\W \sF$ in \eqref{Eq:IEWassGradient}, for every $x,y \in \mR^d$ and $\mu, \nu \in \sP_2(\mR^d)$:
\begin{align*}
    &\| \nabla_\W \sF(x,\mu) - \nabla_\W \sF(y,\nu) \| \\
    &\leq  \|\nabla V(x) - \nabla V(y)\| + \|\int \nabla W(x-z)\mu(dz) - \int \nabla W(y-z)\nu(dz)\| \\
    &\leq L_V \|x - y\| \\
    &+ \| \int \nabla W(x-z)\mu(dz) - \int \nabla W(y-z)\mu(dz) + \int \nabla W(y-z)\mu(dz) - \int \nabla W(y-z)\nu(dz) \| \\
    &= (L_V + L_W) \|x - y\| + \E_{(Z_1,Z_2)\sim \mu\times \nu}\|\nabla W(y-Z_1) - \nabla W(y-Z_2) \| \\
    &= (L_V + L_W) \|x - y\| + L_W \W_1(\mu, \nu) \\
    &\leq (L_V + L_W) \|x - y\| + L_W \W_2(\mu, \nu) \,.
\end{align*}

Similarly, we consider $\nabla_{\W}\bar{\sF}(x,\mu) = \nabla V(x) + \nabla W\ast \pi (x)$ and conclude $L_{\bar{u}} = L_V + L_{W}$, $L_{\bar{\sF}} =0$ and $L_l = L_W$. Clearly, Assumption~\ref{Assumption:functional}-\eqref{eq:LSI} is satisfied with constant $C_{\mathsf{LSI}}$ since it is the same as Assumption~\ref{assumption:pairwise_lsi}. Notice that 
$\bar{\sE}(\mu) := \frac{\sigma^2}{2}\KL(\mu || \pi)$, and that $\nabla_{\W}\KL(\mu||\pi) = \FD(\mu||\pi)$, we conclude that Assumption~\ref{assumption:pairwise_lsi} implies Assumption~\ref{Assumption:functional}-\eqref{eq:PL} with $C_{\bar{\sE}} = \frac{\sigma^2}{C_{\mathsf{LSI}}}$. The choice of $\bar{\sE}$ also implies Assumption~\ref{Assumption:functional}-\eqref{eq:kl_growth} with $C_{\KL} = \frac{2}{\sigma^2} $.

Now consider  Assumption~\ref{Assumption:BoundVariance}. Note that by Jensen's inequality:
\begin{align*}
    \E_{Y\sim \rho_0} \| u(x,Y) + \nabla_\W \sF(x;\rho_0) \|^2 &= \E_{Y\sim \rho_0} \| \nabla W(x - Y) - \nabla W \ast \rho_0(x) \|^2 \\
    &\leq \E_{Y\sim \rho_0}\int \| \nabla W(x-Y) - \nabla W(x-z)\|^2 \rho_0(dz) \\
    &=  L_W^2\E_{(Y,Z) \sim \rho_0\times \rho_0} \| Y - Z \|^2 \\
    &= 2L_W^2\Var(\rho_0) \,.
\end{align*}
Hence, Assumption \ref{Assumption:BoundVariance} is satisfied with $C^{\Var} = 2L_W^2$ and $C^{\Var}_{\nu} = 0$.
\end{proof}

\subsection{Proof of Theorem~\ref{thm:main_pairwise}}
\label{subsec:main_pairwise_proof}
\begin{proof}
Using the notation in Theorem~\ref{thm:main_theorem}, we conclude that $\bar{\sF}(x,\pi) = \sF(x,\pi)$ for every $x \in \R^d$ and hence $G_{\pi} = G_{\mathsf{mod}}$. Since $\pi \propto \exp(- \frac{2\delta\sF(x,\pi)}{\sigma^2})$ and $\nabla_\W\sF(x,\pi) = \nabla_x \delta \sF(x,\pi)$, we apply Lemma~\ref{lem:energy_gradient} to conclude that $(G_{\pi})^2 := \E_{x\sim \pi}\|\nabla_\W \sF(x,\pi)\|^2 \leq \frac{d\sigma^2 (L_V+L_W) }{2}$. Applying Lemma~\ref{lem:stationarity_bounds}, we conclude $\Var(\pi) \leq C_{\mathsf{LSI}}d$. Using Assumption~\ref{Assumption:BoundVariance} instantiated to our case, we conclude $(\sigma^*)^2 \leq 2L_W^2C_{\mathsf{LSI}}d$. Under the assumption on $\eta$, all the requisite assumptions for Theorem~\ref{thm:main_theorem} are satisfied (after simplifying with the fact that $C_{\mathsf{LSI}} \geq \frac{\sigma^2}{2(L_V+L_W)}$ from Lemma~\ref{lem:lsi_lb}). Thus, we note that $\gamma_1,\gamma_2,\gamma_3$ in Theorem~\ref{thm:main_theorem} can be instantiated as 
 $$\gamma_1 \lesssim \sigma d^{\tfrac{3}{2}}\sqrt{C_{\mathsf{LSI}}}L_W(L_V+L_W)\,;\, \gamma_2 \lesssim (L_V+L_W)^2\sigma^2 d^2\,;\, \gamma_3 \lesssim (L_V+L_W)^3d\sigma^2 \,.$$

Here, we have Assumption~\ref{assumption:pairwise_linearize} that $L_W \leq \frac{\sigma^2}{4C_{\mathsf{LSI}}} $ to simplify the expressions. Invoking Theorem~\ref{thm:main_theorem}, using the fact that $\eta (L_V+L_W) < c_0$ by assumption in the statement of this theorem, and $C_{\mathsf{LSI}} \geq \frac{\sigma^2}{2(L_V+L_W)}$ from Lemma~\ref{lem:lsi_lb}, we conclude the result.

\end{proof}

\section{Proof for Mean Field Neural Networks}
\subsection{Proof of Lemma~\ref{lem:assumption_mean_field}}
\label{subsec:assumption_mean_field_proof}
\begin{proof}
 Note that $u = \hat{G}$ satisfies the Lipschitz continuity property of $x \to u(x,y,\xi)$ and $y \to u(x,y,\xi)$ in Assumption \ref{Assumption:Lipschitz} with $L_u = (B + R) (LR + N) + \lambda + M^2R^2$ because for every $x_1,x_2 \in \mR^d$:
\begin{align*}
    \| u(x_1,y,i) - u(x_2,y,i) \| &\leq (h(z_i,Y) - w_i) \|\nabla_x h(z_i,x_1) - \nabla_x h(z_i,x_2)\| + \lambda \| x_1 - x_2 \| \\
    &\leq (h(z_i,y) - w_i) (L\|z_i \|) \| x_1 - x_2 \| + \lambda \| x_1 - x_2 \| \\
    &\leq \left((B + R)LR + \lambda \right) \| x_1 - x_2 \| \,,
\end{align*}
where the last two inequalities follow from Assumptions  \ref{Assumption:MFNBoundedness} and \ref{Assumption:MFNLipschitz}. Similarly considering $\|u(x,y_1,i) - u(x,y_2,i)\|$, we conclude the result.
Similarly, we can take $L_{\sF} = M^2 R^2$ because, for every $x,y \in \mR^d$ and $\mu, \nu \in \sP_2(\mR^d)$:
\begin{align*}
    &\| \nabla_\W \sF(x,\mu) - \nabla_\W \sF(x,\nu) \| \\
    &= \bigg\|\frac{1}{m} \sum_{i=1}^m \left[\left(\int h(z_i,w)d\mu(w) - \int h(z_i,w)d\mu(v)\right)\nabla_x h(z_i,x)\right]\bigg\| \\
    &\leq \frac{1}{m}\sum_{i=1}^{m}\|\nabla_x h(z_i,x)\|M\|z_i\|\W_1(\mu,\nu)\\
    &\leq M^2R^2\W_2(\mu,\nu) \,.
\end{align*}

Next, since the proximal Gibbs measure satisfies the LSI by Assumption \ref{Assumption:MFNLSI}, the LSI condition in Assumption \ref{Assumption:functional} is satisfied with $C_{\mathsf{LSI}} = C_{\mathsf{LSI}}$. By \cite[Proposition 1]{nitanda2022convex}, we conclude that $C_{\KL} = \frac{2}{\sigma^2}$. Again, by \cite[Proposition 1]{nitanda2022convex}, we have: $\nabla_{\W} \bar{\sE}(x,\mu) = \frac{\sigma^2}{2} \nabla_x \log(\frac{\mu(x)}{\pi_{\mu}(x)})$. Therefore, have:
\begin{align}
\int \|\nabla_{\W}\bar{\sE}(x,\mu)\|^2d\mu(x) &= \frac{\sigma^4}{4}\FD(\mu||\pi_{\mu}) \nonumber \\
&\geq \frac{\sigma^4}{2C_{\mathsf{LSI}}}\KL(\mu||\pi_{\mu}) \tag{By Assumption~\ref{Assumption:MFNLSI}}\nonumber \\
&\geq \frac{\sigma^2}{C_{\mathsf{LSI}}} \bar{\sE}(\mu) \tag{By \cite[Proposition 1]{nitanda2022convex}} \,.
\end{align}

Therefore, we conclude that Assumption~\ref{Assumption:functional}-\eqref{eq:PL} holds with $C_{\bar{\sE}} = \frac{\sigma^2}{C_{\mathsf{LSI}}}$. Let $Y,I \sim \rho_0\times \nu$ and let $Y',I'$ be an independent copy of $Y,I$. By Assumptions \ref{Assumption:MFNBoundedness}, \ref{Assumption:MFNLipschitz} we have $\|u(x,y,i) + \nabla_{\W}\sF(x,\rho_0)\| \leq 2(B+R)MR$ almost surely. Thus, the mean-field neural network case satisfies Assumption \ref{Assumption:BoundVariance} with $C^{\Var} = 0$ and $C^{\Var}_\nu = 4(B+R)^2M^2R^2$.

\end{proof}

\subsection{Proof of Theorem~\ref{thm:main_mean_field}}
\label{subsec:main_mean_field_proof}
\begin{proof}
      Under the parameter correspondence established in Lemma~\ref{lem:assumption_mean_field}, with our choice of $\eta$, we conclude that the conditions for Theorem~\ref{thm:main_theorem} are satisfied (once considered with the fact that $C_{\mathsf{LSI}} \geq \frac{\sigma^2}{2L_u}$ from Lemma~\ref{lem:lsi_lb}). Since $\pi \propto \exp(- \frac{2\delta\sF(x,\pi)}{\sigma^2})$ and $\nabla_\W\sF(x,\pi) = \nabla_x \delta \sF(x,\pi)$, we apply Lemma~\ref{lem:energy_gradient} to conclude that $(G_{\pi})^2 = G_{\mathsf{mod}}^2:= \E_{x\sim \pi}\|\nabla_\W \sF(x,\pi)\|^2 \leq \frac{d\sigma^2 L_u }{2}$. Using Assumption~\ref{Assumption:BoundVariance} instantiated to our case, we conclude $(\sigma^*)^2 \leq 4M^2R^2(B+R)^2$.
    
    Instantiating the quantities $\gamma_1,\gamma_2$ and $\gamma_3$ found in Theorem~\ref{thm:main_theorem} to the case of mean field neural networks, we have:
    \[
    \gamma_1 \lesssim \sigma d L_u M R (B+R) 
        \,,\,\gamma_2 \lesssim \sigma^2 L^2_u d + L_uM^2R^2(B+R)^2 
        \,,\,\gamma_3 \lesssim \sigma^2 L^3_u d + L_u^2M^2R^2(B+R)^2 \,.
    \]
    We then apply Theorem~\ref{thm:main_theorem} and simplify using the fact that $\eta L_u \leq c_0$, and $C_{\mathsf{LSI}} \geq \frac{\sigma^2}{2L_u}$ from Lemma~\ref{lem:lsi_lb} to conclude the result. 

\end{proof}

\end{document}